\documentclass{article}

\usepackage[accepted]{icml2020}

\usepackage{amsfonts}
\usepackage{amsmath}
\usepackage{amsthm}
\usepackage{amssymb}
\usepackage{bm}
\usepackage{bbm}
\usepackage[T1]{fontenc}
\usepackage{graphicx}
\usepackage[labelfont=bf]{caption}
\usepackage{dblfloatfix}

\usepackage{xcolor}
\usepackage{colortbl}
\usepackage{multirow,makecell}
\usepackage{pifont}
\usepackage{booktabs}
\usepackage{thm-restate}
\usepackage{enumitem}
\usepackage{wrapfig}
\usepackage{url}
\usepackage{graphicx}
\usepackage{microtype}
\usepackage{subfigure}
\usepackage{hyperref}

\usepackage{natbib}

\usepackage{algorithm}
\usepackage[noend]{algpseudocode}
\algrenewcommand\algorithmicindent{0.173in}
\usepackage{float}
\newfloat{algorithm}{t}{lop}
\floatname{algorithm}{Algorithm}

\newcommand{\minimizel}[1]{\mathrm{minimize}_{#1} \ }

\newcommand{\ip}[1]{\langle #1 \rangle}
\newcommand{\ceil}[1]{\lceil #1 \rceil}
\newcommand{\floor}[1]{\lfloor #1 \rfloor}
\newcommand{\norm}[1]{\left\| #1 \right\|}
\newcommand{\norms}[1]{\| #1 \|}
\newcommand{\E}[1]{\mathbb{E}\left[ #1 \right]}
\newcommand{\Eu}[2]{\mathbb{E}_{#1}\left[ #2 \right]}
\newcommand{\Eus}[2]{\mathbb{E}_{#1}[ #2 ]}
\newcommand{\Es}[1]{\mathbb{E}[ #1 ]}

\newcommand{\algcomment}[1]{\State \emph{\# #1:} \vskip 0.01in}

\definecolor{gray}{HTML}{AAAAAA}
\definecolor{mgray}{HTML}{888888}
\definecolor{lgray}{HTML}{FFFFFF}
\definecolor{dgreen}{HTML}{084F28}
\definecolor{dred}{HTML}{8B0000}

\newcommand{\gr}[1]{\textcolor{mgray}{#1}}

\def\reals{\mathbb{R}}

\def\cifar{$\mathtt{cifar10}$}
\def\imagenet{$\mathtt{imagenet}$}
\def\deepspeech{$\mathtt{speech}$}
\def\yolo{$\mathtt{yolo}$}
\def\transformer{$\mathtt{transformer}$}

\newcommand{\m}[1]{\mathbf{#1}}

\def\w{\m{w}}
\def\xtil{\tilde{\x}}
\def\ftil{\tilde{f}}

\def\wtp{\w_{t+1}}

\def\wz{\w_{0}}
\def\wtm{\w_{t-1}}

\def\wt{\w_{t}}
\def\wT{\w_{T}}

\def\MBt{\mu_{\g}^2(\wt)}
\def\VB{\sigma_{\g}^2(\w)}
\def\VBtm{\sigma_{\g}^2(\wtm)}
\def\SB{\bm{\Sigma}_{\g}(\w)}

\def\VBt{\sigma_{\g}^2(\wt)}

\def\gbart{\bar{\g}_{t}}

\def\xt{\x_{t}}

\def\xibt{\bm{\xi}_t}

\def\algname{Ada\-Scale}
\def\fullalgname{Ada\-Scale SGD}
\def\shortalgname{Ada\-Scale}
\def\algline{\vskip -0.06in \hspace*{0.0in}{\color{gray} \hrulefill} \vskip -0.08in}
\def\falgline{\vskip -0.1in \hspace*{-0.395in}{\hrulefill} \vskip -0.08in}

\def\etat{{\eta_{t}}}

\def\etasq{\eta^2}

\def\g{\m{g}}
\def\gt{\g_{t}}
\def\gtm{\bar{\g}_{t-1}}

\def\rt{r_{t}}
\def\rtp{r_{t'}}
\def\rtm{r_{t-1}}

\def\TB{T_{\mathrm{SI}}}
\def\TSo{T_{S1}}

\def\DeltaB{\Delta}

\def\x{\m{x}}
\def\X{\mathcal{X}}

\def\taut{\tau_{t}}
\def\tautp{\tau_{t+1}}

\def\tauz{\tau_{0}}

\def\Ft{\tilde{F}}
\def\LSS{LSW+}
\def\vsp{\vspace{-0.1em}}
\def\vspf{\vspace{-0.5em}}

\def\tlp{\textrm{(}}
\def\trp{\textrm{)}}
\newcommand{\lrsched}[1]{\mathtt{lr}\tlp #1 \trp}

\def\lrschedname{\mathtt{lr}}
\def\Slrschedname{\mathtt{lr}_S}
\def\ST{T_S}
\def\Olrschedname{\mathtt{lr}_1}
\def\OT{T_1}

\renewcommand{\lim}[1]{\underset{#1}{\mathrm{lim}}}

\newcommand{\eqnref}[1]{(\ref{#1})}
\newcommand{\prbref}[1]{Problem \ref{#1}}
\renewcommand{\algref}[1]{Algorithm~\ref{#1}}
\newcommand{\figref}[1]{Figure~\ref{#1}}
\newcommand{\secref}[1]{\S\ref{#1}}
\newcommand{\appref}[1]{Appendix~\ref{#1}}
\newcommand{\apprefs}[1]{Appendix~\ref{#1}}
\newcommand{\tblref}[1]{Table~\ref{#1}}
\newcommand{\assref}[1]{Assumption~\ref{#1}}
\newcommand{\thmref}[1]{Theorem~\ref{#1}}
\newcommand{\lemref}[1]{Lemma~\ref{#1}}

\newcommand{\propref}[1]{Proposition~\ref{#1}}

\theoremstyle{plain}

\newtheorem{dfn}{Definition}
\newtheorem{ass}{Assumption}

\newtheorem{lem}{Lemma}

\makeatletter
\def\@fnsymbol#1{\ensuremath{\ifcase#1\or \dagger\or \ddagger\or
\mathsection\or \mathparagraph\or \|\or **\or \dagger\dagger
\or \ddagger\ddagger \else\@ctrerr\fi}}
\makeatother

\icmltitlerunning{\fullalgname}

\begin{document}

\twocolumn[
\icmltitle{
    \fullalgname{}: 
    A User-Friendly
    Algorithm for Distributed Training
}



\icmlsetsymbol{equal}{$\dagger$}

\begin{icmlauthorlist}
\icmlauthor{Tyler B. Johnson}{equal,apple}
\icmlauthor{Pulkit Agrawal}{equal,apple}
\icmlauthor{Haijie Gu}{apple}
\icmlauthor{Carlos Guestrin}{apple}
\end{icmlauthorlist}

\icmlaffiliation{apple}{Apple, Seattle, WA}

\icmlcorrespondingauthor{T. Johnson}{tbjohns@apple.com}
\icmlcorrespondingauthor{P. Agrawal}{pulkit\_agrawal@apple.com}

\icmlkeywords{Machine Learning, ICML}

\vskip 0.3in
]



\printAffiliationsAndNotice{\icmlEqualContribution} % otherwise use the standard text.

\begin{abstract}
When using large-batch training to speed up stochastic gradient descent, learning rates must adapt to new batch sizes in order to maximize speed-ups and preserve model quality.  Re-tuning learning rates is resource intensive, while fixed scaling rules often degrade model quality.  We propose \fullalgname{}, an algorithm that reliably adapts learning rates to large-batch training.  By continually adapting to the gradient's variance, \algname{} automatically achieves speed-ups for a wide range of batch sizes.  We formally describe this quality with \algname{}'s convergence bound, which maintains final objective values, even as batch sizes grow large and the number of iterations decreases.  In empirical comparisons, \algname{} trains well beyond the batch size limits of popular ``linear learning rate scaling'' rules.  This includes large-batch training with no model degradation for machine translation, image classification, object detection, and speech recognition tasks.  \algname{}'s qualitative behavior is similar to that of ``warm-up'' heuristics, but unlike warm-up, this behavior emerges naturally from a principled mechanism.  The algorithm introduces negligible computational overhead and no new hyperparameters, making \algname{} an attractive choice for large-scale training in practice.
\end{abstract}

\section{Introduction} \label{sec:intro}

Large datasets and large models underlie
much of the recent success of
machine learning. 
Training such models
is time consuming, however, often requiring days or even weeks.
Faster training enables consideration of more data and models,
which expands the capabilities of machine learning.

Stochastic gradient descent and its variants are fundamental
training algorithms.
During each iteration, SGD applies a small and noisy update to the
model, based on a stochastic gradient.
By applying thousands of such updates in sequence,
a powerful model can emerge over time.

To speed up SGD,
a general strategy is to improve
the signal-to-noise ratio of each update, i.e.,
reduce the variance of the stochastic gradient.
Some tools for this include
SVRG-type gradient estimators \citep{Johnson:2013,Defazio:2014}
and prioritization of training data via
importance sampling \citep{Needell:2014,Zhao:2015}.
Data parallelism---our focus in this work---is perhaps
the most powerful of such tools.
By processing thousands of training examples
for each gradient estimate, distributed systems can
drastically lower the gradient's variance.
Only the system's scalability, not the 
algorithm, limit the variance reduction.

Smaller variances alone, however,
typically result in unimpressive speed-ups.
To fully exploit the noise reduction, SGD must also
make larger updates, i.e.,
increase its \emph{learning rate} parameter.
While this fact applies universally 
(including for SVRG \citep{Johnson:2013}, 
data parallelism \citep{Goyal:2017}, 
and importance sampling \citep{Johnson:2018}),
the precise relationship between
variance and learning rates remains poorly quantified.

For this reason, practitioners often turn
to heuristics in order to adapt learning rates.
This applies especially to distributed training,
for which case ``fixed scaling rules''
are standard but unreliable strategies.
For example, \citet{Goyal:2017} popularized
``linear learning rate scaling,''
which succeeds 
in limited cases
\citep{Krizhevsky:2014,Devarakonda:2017,Jastrzebski:2018,Smith:2018,Lin:2019}.
For other problems or greater scales, however, linear scaling
leads to poor model quality and even divergence---a result
known both in theory \citep{Yin:2018,Jain:2018,Ma:2018} and in practice \citep{Goyal:2017}.

In this work, we introduce \fullalgname{},
an algorithm that more reliably adjusts learning rates
for large-batch training.
\algname{} achieves 
speed-ups that depend naturally 
on the gradient's variance before it is decreased---nearly perfect 
linear speed-ups
in cases of large variance, and smaller speed-ups otherwise.
For contrast, we also show that as batch sizes increase,
linear learning rate scaling progressively degrades model
quality, even when including
``warm-up'' heuristics \citep{Goyal:2017} and additional training iterations.

\begin{figure*}[t]
\begin{center}
\small{
\begin{tabular}{
    @{}
    >{\raggedleft\arraybackslash}p{0.10in}
    @{\hspace{0.02in}}
    c
    @{\hspace{0.05in}}
    >{\raggedleft\arraybackslash}p{0.10in}
    @{\hspace{0.02in}}
    c
    @{\hspace{0.05in}}
    >{\raggedleft\arraybackslash}p{0.10in}
    @{\hspace{0.02in}}
    c
    @{}|@{}
    >{\raggedleft\arraybackslash}p{0.165in}
    @{\hspace{0.02in}}
    c
    @{}
    }
    \multicolumn{6}{c|}{\hspace{0.4in} \algname{}} & \multicolumn{2}{c}{\hspace{0.22in} Linear scaling}
\\[-0.2em]
\rotatebox{90}{\hspace{0.16in} Val. Acc (\%)} &
\includegraphics[width=1.265025in]{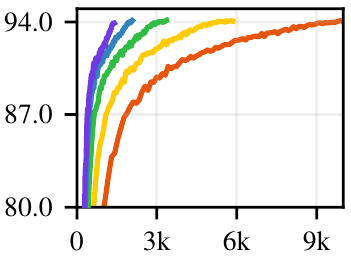} &
\rotatebox{90}{\hspace{0.16in} Val. Acc (\%)} &
\includegraphics[width=1.265025in]{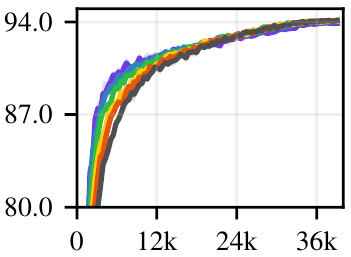} &
\rotatebox{90}{\hspace{0.10in} Learning rate $\etat$} &
\includegraphics[width=1.2149in]{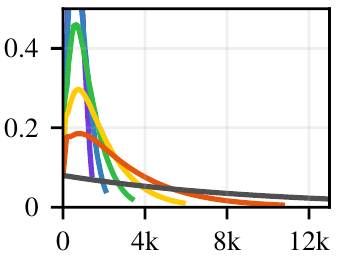} 
\hspace{0.025in}
& 
\rotatebox{90}{\hspace{0.16in} Val. Acc (\%)} &
\includegraphics[width=1.265025in]{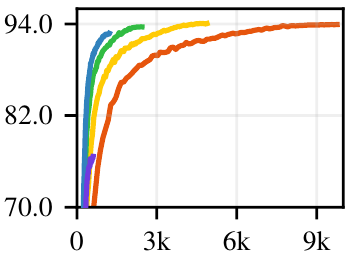} \\[-0.23em]
&  \hspace{0.22in}{Iteration} &
&  \hspace{0.25in}{Scale-invariant itr.} & 
&  \hspace{0.22in}{Iteration} &
& \hspace{0.22in} {Iteration} \\[-0.1em]
\multicolumn{8}{@{}c@{}}{
\includegraphics[width=4.7in]{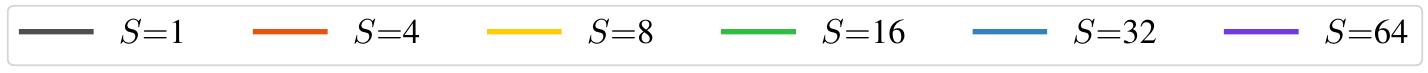} 
} \\[-1em]
\end{tabular}
}
\end{center}
\caption{\textbf{Motivating results.}
For \cifar{},
\algname{} 
automatically adapts to many scales $S$ (the number of parallel 
batches),
preserving model quality.
  When plotted in terms of ``scale-invariant'' iterations,
  training curves align closely.
  With \algname{},
  warm-up behavior emerges naturally when adapting
  a simple learning rate schedule (exponential decay) to scale $S$ (plot 
  cropped to show behavior).
  Meanwhile, linear scaling (with warm-up)
  degrades model quality as batch sizes increase.
}
\label{fig:cifar10_intro}       
\vspf
\end{figure*}

\algname{} makes
large-batch training significantly more user-friendly.
Without changing the algorithm's inputs (such as learning rate schedule),
\algname{} can adapt to vastly different batch sizes with
large speed-ups and near-identical model quality.
This has two important implications:
(i)~\algname{} improves the translation of learning rates
to greater amounts of parallelism, which simplifies scaling up
tasks or adapting to dynamic resource availability; and (ii)~\algname{}
works with simple learning rate schedules at scale,
eliminating the need for warm-up.
Qualitatively,
\algname{} and warm-up have similar effects on learning rates,
but with \algname{}, the behavior emerges from a
principled and adaptive strategy, not hand-tuned parameters.

We perform large-scale empirical evaluations
on five training benchmarks.
Tasks include image classification,
machine translation, object detection, and speech recognition.
The results align remarkably well with our theory, as  \algname{}
systematically preserves model quality across many batch sizes.
This includes training ImageNet with batch size 32k
and Transformer with 262k max tokens per batch.

The ideas behind \algname{} are not limited to distributed training,
but instead apply to any estimator that reduces the gradient's variance.
To provide context for
the remaining sections, \figref{fig:cifar10_intro} includes results from
an experiment using CIFAR-10 data.
These results illustrate \algname{}'s scaling behavior,
the qualitative impact on learning rates,
and a failure case for the linear scaling rule.

\section{Problem formulation}
\label{sec:formulation}

We focus on quickly and approximately solving
\begin{equation} \label{prb:general} 
\minimizel{\w \in \reals^d} F(\w) := \mathbb{E}_{\x \sim \mathcal{X}}\left[ f(\w, \x) \right] \, .
\end{equation}
Here $\w$ parameterizes a machine learning model, while $\mathcal{X}$ denotes a distribution
over batches of training data.
We assume that $F$ and $f$ are differentiable
and that $\mathbb{E}_{\x \sim \mathcal{X}}\left[ \nabla_{\w} f(\w, \x) \right] = \nabla F(\w)$.
  
SGD is a fundamental algorithm for solving \prbref{prb:general}.
Let $\wt$ denote the model parameters when iteration $t$ begins.
SGD samples
a batch $\xt \sim \X$ and computes the gradient $\gt \gets \nabla_\w f(\wt, \xt)$,
before applying the update
\begin{equation} \label{eqn:sgd_update}
 \wtp \gets \wt - \etat \gt \, . 
\end{equation}
Here $\etat$ is the \emph{learning rate}.
Given a learning rate schedule
$\lrschedname \, : \, \mathbb{Z}_{\geq 0} \rightarrow \reals_{> 0}$,
SGD defines 
$\etat = \lrsched{t}$.
For our experiments in \secref{sec:empirical}, $\lrschedname$ is an
exponential decay or step decay function.
SGD completes training  after $T$ iterations.

To speed up SGD, we can process multiple batches in parallel.
\algref{alg:sgd} defines a ``scaled SGD'' algorithm.
At scale $S$,
we sample $S$ independent batches per iteration.
After computing the gradient for each batch in parallel, the algorithm synchronously applies
the mean of these gradients (in place of $\gt$ in \eqnref{eqn:sgd_update}) when updating the model.

As $S$ increases, scaled SGD generally requires fewer iterations to train a model.
But how much speed-up should we expect, especially as $S$ becomes large?
Moreover, how do we adapt learning rates in
order to realize this speed-up?
Practitioners usually answer these questions with one or two approaches: (i)~trial-and-error
parameter tuning, which requires significant time and specialized knowledge, or 
(ii)~fixed scaling rules, which work well for some problems, but result in poor model quality for many others.

\section{\fullalgname{} algorithm} \label{sec:gain_scaling}

In this section, we introduce \algname{}.
We can interpret \algname{}
as a per-iteration interpolation
between two simple rules for scaling the learning rate.
For each of these rules,
we first consider an extreme problem setting
for which the rule behaves ideally.
With this understanding, we then define \algname{} and provide theoretical results.
Finally, we discuss approximations needed for a practical implementation.

\subsection{Intuition from extreme cases of gradient variance} \label{sec:interpolation}

\begin{figure*}
\begin{minipage}[t]{0.495\textwidth}
\input{algorithms/base_sgd}
\end{minipage}
\hfill
\begin{minipage}[t]{0.495\textwidth}
\input{algorithms/gain_scaling}
\end{minipage}
\end{figure*}

To help introduce \algname{},
we first consider two
simple ``scaling rules.'' A scaling rule translates training
parameters (such as learning rates) to large-batch settings:
\begin{dfn}
Consider a 
learning rate schedule $\Olrschedname$ and total steps $\OT$ for
single-batch training (i.e., \algref{alg:sgd} with $S=1$).
Given a scale $S$,
a \textbf{scaling rule} maps $(S, \Olrschedname, \OT)$
to parameters $(\Slrschedname, \ST)$ for
training at scale $S$.
\end{dfn}
One scaling rule is identity scaling,
which keeps training parameters unchanged for all $S$:
\begin{dfn}
The \textbf{identity scaling rule}
defines 
$\ST = \OT$ and
$\Slrschedname = \Olrschedname$.
\end{dfn}
Note that this rule has little practical appeal, 
since it fails to reduce the number of training iterations.
A more popular scaling rule is linear learning rate scaling:
\begin{dfn}\label{dfn:ls}
The \textbf{linear scaling rule} 
defines
$\ST = \ceil{\OT / S}$ and
$\Slrschedname(t) = S \cdot \Olrschedname(St)$.
\end{dfn}
Conceptually, linear scaling treats SGD
as a perfectly parallelizable algorithm.
If true,
applying updates
from $S$ batches in parallel achieves the same result
as doing so in sequence.

For separate special cases of \prbref{prb:general},
the identity and linear scaling rules 
maintain training effectiveness for all $S$.
To show this, we first define the
gradient variance quantities
\[
\begin{array}{c}
\SB = \mathrm{cov}_{\x \sim \X}(\nabla_{\w} f(\w, \x), \nabla_{\w} f(\w, \x)) \, , \\ \\[-0.8em]
\text{and} \quad \VB = \mathrm{tr}(\SB) \, .
\end{array}
\]
In words, $\VB$
sums the variances of each entry in $\nabla_{\w} f(\w, \x)$.
Data parallelism fundamentally impacts SGD by reducing this variance.

We first consider the case of deterministic gradients, i.e., $\VB = 0$.
Here identity scaling preserves model quality:
\begin{restatable}[Identity scaling for zero variance]{prop}{identityprop}
\label{prop:identity}
Let $\w^{(1)}$ denote the result
of single batch training, and let $\w^{(S)}$
denote the result of scaled training
after identity scaling.
If $\VB = 0$ for all $\w \in \reals^d$, then
$F(\w^{(1)}) = F(\w^{(S)})$.
\end{restatable}
Although identity scaling does not speed up
training, \propref{prop:identity}
is important for framing the impact of increasing batch sizes.
If the gradient variance is ``small,'' then we
cannot expect large gains from scaling up training---further reducing
the variance
has little effect on $\gbart$.
With ``large'' variance, however, the opposite is true:

\begin{restatable}[Linear scaling for large variance]{prop}{linearprop} \label{prop:linear}
Consider any learning rate $\eta$
and training duration $T$.
For some fixed covariance matrix $\bm{\Sigma} \in \mathbb{S}^d_{>0}$
and $\nu \in \mathbb{Z}_{> 0}$,
assume $\nabla_{\w} f(\w, \x) \sim \mathcal{N}(\nabla F(\w), \nu \bm{\Sigma})$.
Let $\Olrschedname(t) = \nu^{-1} \eta$
and $\OT = \nu T$.
Define $\w^{(1)}$ as the result
of single batch training and $\w^{(S)}$
as the result of scaled training
after linear scaling.
Then
$\E{F(\w^{(1)})} = \E{F(\w^{(S)})}$
in the limit $\nu \rightarrow +\infty$.
\end{restatable}

In simpler terms,
linear scaling leads to perfect linear speed-ups
in the case of
very large gradient variance (as well as small learning rates
and many iterations, to compensate for this variance).
Since 
increasing $S$
decreases variance,
it is natural
that scaling
yields large speed-ups in this case.

In practice, 
the gradient's variance is neither zero nor infinite,
and both identity and linear scaling
may perform poorly.
Moreover, the gradient's variance
does not remain constant throughout training.
A practical algorithm, it seems, must continually adapt to the state of training.

\subsection{\algname{} definition} \label{sec:def}

\algname{}, defined in \algref{alg:gs}, adaptively interpolates between 
identity and linear scaling, based on the expectation of the gradient's variance.
During iteration $t$, \algname{} multiplies the learning
rate by a ``gain ratio'' $\rt \in [1, S]$:
\[ \etat = \rt \cdot \lrsched{\floor{\taut}} \, .  \]
Here we define
$\taut = \sum_{t' = 0}^{t - 1} \rtp$---the
\emph{scale-invariant iteration}.
The idea is that iteration $t$
performs the equivalent of $\rt$ single-batch iterations,
and $\tau_t$ accumulates this progress.
\algname{}
concludes when $\taut \geq \TB$,
where $\TB$
is the ``total scale-invariant iterations.''
Since $\rt \in [1, S]$, \algname{} requires
at least $\ceil{\TB / S}$ and
at most $\TB$ iterations.
For practical problem settings, this training duration 
often falls closer to $\ceil{\TB / S}$ than $\TB$ (as we see later in \secref{sec:empirical}).

The identity and linear rules correspond to two special cases of \algname{}.
If $\rt = 1$ for all $t$,
the algorithm equates to SGD with identity scaling.  
Similarly, if $\rt = S$ for all $t$,
we have linear scaling.
Thus, 
\secref{sec:interpolation} 
suggests setting
$\rt \approx 1$
when $\VBt$ is small
and
$\rt \approx S$
when this variance is large.
Introducing the notation
$\MBt = \norm{\nabla F(\wt)}^2$,
\algname{} achieves this by defining 
\begin{equation} \label{eqn:gain_ratio_def}
\rt = \frac{ \Eu{\wt}{ \VBt + \MBt  } }{ \Eu{\wt}{\tfrac{1}S \VBt + \MBt} }
 \, .
\end{equation}
Here there are expectations both with respect to the
batch distribution $\X$ (as part of the variance term, $\VBt$)  and with respect to
the distribution over training trajectories.
A practical implementation requires estimating $\rt$, and we describe a procedure
for doing so in \secref{sec:estimating_noise}.
Interestingly, since $\Eu{\wt}{\MBt}$ decreases toward zero with progress toward convergence,
we find empirically that $\rt$ often gradually increases during
training, resulting in a ``warm-up'' effect on the learning rate (see \secref{sec:empirical}).
In addition to this intuitive behavior, \eqnref{eqn:gain_ratio_def}
has a more principled justification.
In particular, 
this $\rt$
ensures that as $S$ increases, $\etat \Es{\MBt}$
and $\etat^2 \Es{\norms{\gbart}^2}$
increase multiplicatively by $\rt$.
This leads to the convergence bound for \algname{} that we next present.

\begin{table*}[!b]
\vsp
\vsp
\vspf
\caption{\textbf{Overview of training benchmarks.}} \label{tbl:tasks}
\vspf
\vspf
\small{
\begin{center}
\begin{tabular}{lllll} 
\toprule
Name & {Task} & {Model} & {Dataset} & {Metric} \\
\midrule
\cifar{} & Image classification & ResNet-18 (v2) & CIFAR-10 & Top-1 accuracy (\%) \\ 
\imagenet{} & Image classification & ResNet-50 (v1) & ImageNet & Top-1 accuracy (\%) \\ 
\deepspeech{} & Speech recognition & Deep speech 2 & LibriSpeech & Word accuracy (\%) \\
\transformer{} & Machine translation & Transformer base & WMT-2014 & BLEU\\  
\yolo{} & Object detection & YOLOv3 & PASCAL VOC & mAP (\%) \\
\bottomrule
\end{tabular}
\end{center}
}
\end{table*}

\subsection{Theoretical results} \label{sec:theory}

We now present convergence bounds
that describe
the speed-ups from \algname{}.
Even as the batch size increases,
\algname{}'s bound converges to the same objective value.
We also include an analogous bound for linear scaling.

Let us define
\mbox{$F^* = \min_{\w} F(\w)$}.
Our analysis requires a few assumptions
that are typical of SGD analysis of non-convex problems
(see, for example, \citep{Lei:2017,Yuan:2019}):
\begin{ass}[$\alpha$-Polyak-\L{}ojasiewicz] \label{ass:pl}
For some $\alpha > 0$, 
$F(\w) - F^*  \leq \tfrac{1}{2 \alpha} \norm{\nabla F(\w)}^2$
for all $\w$.
\end{ass}
\begin{ass}[$\beta$-smooth] \label{ass:smooth}
For some $\beta > 0$, we have
$\norms{\nabla F(\w) - \nabla F(\w')} \leq \beta \norms{\w - \w'}$ for all $\w$, $\w'$.
\end{ass}
\begin{ass}[Bounded variance] \label{ass:variance}
There exists a $V \geq 0$ such that 
$\VB \leq V $ for all $\w$.
\end{ass}

We emphasize that we do not assume
convexity. The PL condition,
which is perhaps the strongest of the assumptions,
is proven to hold for some nonlinear neural networks
\citep{Charles:2018}.

We consider constant $\lrschedname$ schedules,
which result in simple and instructive bounds.
Furthermore, constant learning rate schedules
provide reasonable results for many deep learning problems \citep{Sun:2019}.
To provide context for the \algname{} bound,
we first present a result for single-batch training:
\begin{restatable}[Single-batch SGD bound]{thm}{basebound} \label{thm:base_bound}
Given Assumptions~\ref{ass:pl}, \ref{ass:smooth}, \ref{ass:variance} and $\eta \in (0, 2 \beta^{-1})$,
consider \algref{alg:sgd}
with $S = 1$ and \emph{$\lrsched{t} = \eta$}. Defining
$\gamma = \eta \alpha (2 - \eta \beta)$
and $\DeltaB = \tfrac{1}{2 \gamma} \eta^2 \beta V$, we have
\vsp
\vsp
\vsp
\[ \E{F(\wT) - F^*} \leq (1 - \gamma)^T \left[F(\wz) - F^* \right] + \DeltaB \, .  \]
\end{restatable}
\vsp \vsp
\vsp \vsp
The bound describes two important characteristics of the single-batch algorithm. First, the suboptimality
converges in expectation
to at most $\DeltaB$.
Second, convergence to $\DeltaB + \epsilon$
requires at most $\ceil{\log((F(\w_0)-F^*) \epsilon^{-1}) / \log((1-\gamma)^{-1})}$ iterations.
We note similar bounds exist,
under a stronger variance assumption
\citep{Karimi:2016,Reddi:2016,De:2017,Yin:2018}.

Importantly, our \algname{}
bound converges to this same $\Delta$ for all practical values of $S$:

\begin{restatable}[\algname{} bound]{thm}{gsbound} \label{thm:gs_bound}
Define $\gamma$, $\DeltaB$ as in \thmref{thm:base_bound}. 
Given Assumptions~\ref{ass:pl}, \ref{ass:smooth}, \ref{ass:variance}, 
$S \leq \gamma^{-1}$,
and 
$\eta \in (0, 2 \beta^{-1})$,
define $T$ as the total iterations of \algref{alg:gs}
with \emph{$\lrsched{t} = \eta$} and scale $S$. 
Denoting $\bar{r} = \tfrac{1}T \sum_{t=0}^{T-1} \rt$, we have
\vsp
\vsp
\vsp
\[ 
\E{F(\wT) - F^*} \leq 
(1 - \gamma)^{\bar{r}T} \left[ F(\wz)- F^* \right] + \DeltaB \, .
\]
\end{restatable}
\vsp \vsp
\vsp \vsp
This bound is remarkably similar to that of \thmref{thm:base_bound}.
Like single-batch SGD, the expected suboptimality
converges to at most $\DeltaB$, but
\algname{} achieves this for many batch sizes.
Also, \algname{}
reduces the total training iterations by a factor 
$\bar{r}$, the average gain.
That is, AdaScale requires at most
$\ceil{\bar{r}^{-1} \log((F(\w_0)-F^*) \epsilon^{-1}) / \log((1-\gamma)^{-1})}$
iterations
to converge to objective value $\DeltaB + \epsilon$.

Finally, we provide an analogous bound for linear scaling:
\begin{restatable}[Bound for linear scaling rule]{thm}{lsbound} \label{thm:ls_bound}
Define $\gamma$ and $\DeltaB$ as in \thmref{thm:base_bound}.
Given 
$\eta \in (0, 2 (S \beta)^{-1})$
and 
Assumptions~\ref{ass:pl}, \ref{ass:smooth}, \ref{ass:variance},  
consider \algref{alg:sgd}
with \emph{$\lrsched{t} = S \eta$}.
Defining $\xi(S) = (2 - \eta \beta)/(2 - S \eta \beta)$ and $F_0 = F(\wz)$, we have
\vsp
\vsp
\vsp
\[
\E{F(\wT) - F^*} \leq  
\left(1 - \tfrac{\gamma}{\xi(S)} \right)^{ST} \left[ F_0- F^* \right] + {\xi(S)} \cdot \DeltaB \, .
\]
\end{restatable}
\vsp \vsp
\vsp \vsp
Note $\xi(S) \geq 1$, and this function increases monotonically with $S$
(until an asymptote at $S = 2 (\eta \beta)^{-1}$).
Thus, unlike in \thmref{thm:gs_bound}, the bound converges
to a value that increases with $S$.
This means that compared to \algname{}, linear scaling often leads to worse
model quality
and greater risk of divergence, especially for large $S$.
We observe this behavior throughout our empirical comparisons in \secref{sec:empirical}.

Finally, we note 
\thmref{thm:gs_bound} 
requires $S \leq \gamma^{-1}$---for reasonably small $\eta$,
this is similar to the stricter constraint
$S \leq (2 \eta \alpha)^{-1}$.
Meanwhile,
\thmref{thm:ls_bound} requires that $S \leq 2 (\eta \beta)^{-1}$.
Compared to the constraint for \algname{},
this constraint
is typically much stricter, since $\alpha \leq \beta$ (and
$\alpha \ll \beta$ if the Hessian's eigenspectrum contains large outlier values,
which has been observed in practice for nerual networks \citep{Sagun:2017,Ghorbani:2019}).
It is also worth noting that if $\gamma$ is large enough that $S \geq \gamma^{-1}$,
then training at smaller scales converges quickly (due to \thmref{thm:base_bound}),
and large batch sizes are likely unnecessary.

\subsection{Practical considerations} \label{sec:estimating_noise}

\begin{figure}[t]

\begin{center}
\small{
\begin{tabular}{@{\hspace{0.02in}}p{0.1in}@{}c@{\hspace{0.03in}}c@{\hspace{0.03in}}c@{}}
& \hspace{0.0in} \imagenet, $S$$=$16 &
\imagenet, $S$$=$128 &
\cifar, $S$$=$8
 \\[-0.1em]
\rotatebox{90}{\hspace{0.17in} Gain ratio $\rt$}
& \includegraphics[width=0.97893in]{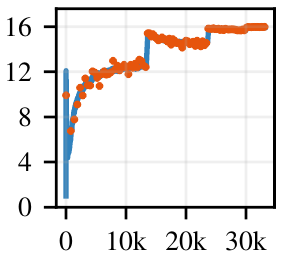} 
& \includegraphics[width=1.02727in]{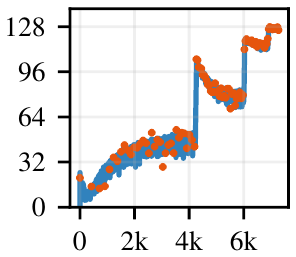} 
& \includegraphics[width=0.93059in]{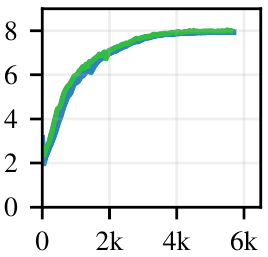}
\\[-0.2em]
& \hspace{0.15in} Iteration $t$
& \hspace{0.15in} Iteration $t$
& \hspace{0.15in} Iteration $t$
\\[-0.1em]
\multicolumn{4}{c}{
\includegraphics[width=1.71in]{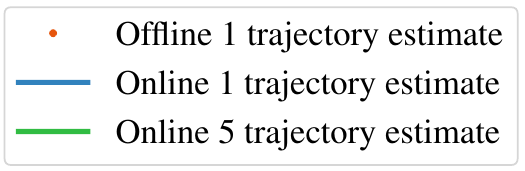} 
}
\end{tabular}
}
\end{center}
\vspace{-1.3em}
\caption{\textbf{Gain ratios}.
Plots compare various $\rt$ estimates.
In practice, \algname{} uses online estimates
from a single training trajectory.
We compare to offline estimates (using 1000 batches).
We also compare to
online estimates that 
use five independent trajectories (tied only
by their shared gain estimates)
to average gradient moment estimates
across trajectories
during each iteration.
The values align closely.
For \imagenet,
abrupt changes align with $\lrschedname$
step changes.
}
\label{fig:gain_ratios}
\vspf
\end{figure}

A practical \algname{} implementation 
must efficiently approximate the gain ratio $\rt$ for all iterations.
Fortunately,
the per-batch gradients
$\gt^{(1)}, \ldots, \gt^{(S)}$
and aggregated gradient $\gbart$
are readily available in distributed SGD algorithms.
This
makes approximating $\rt$ straightforward, since
in addition to \eqnref{eqn:gain_ratio_def}, we can express
$\rt$ as the ratio of expectations
\vspf{}
\[
\rt = \frac{
\Es{\tfrac{1}S \sum_{i=1}^S \norms{\gt^{(i)}}^2}
}{
\Es{\norms{\gbart}^2}
} \, .
\vsp{}
\vsp{}
\vsp{}
\]
Here we take the expectation over all randomness of the algorithm (both
current and prior batches).

To estimate the gain, we recommend tracking moving averages of
$\tfrac{1}S \sum_{i=1}^S \norms{\gt^{(i)}}^2$ and
${\norms{\gbart}^2}$ across iterations.
Denoting these averages by $m_1$ and $m_2$, respectively, 
we can estimate the gain as $\hat{r}_t = \tfrac{m_1 + \epsilon}{m_2 + \epsilon}$.
Here $\epsilon$ is a small constant, such as $10^{-6}$, for numerical stability.

For our empirical results, we use 
exponential moving averages with parameter 
$\theta = \max\{1 - S / 1000, 0\}$,
where $\theta = 0$
results in no averaging.
We find that \algname{} is robust to the choice of $\theta$, and
we provide evidence of this in \appref{app:theta}.
When ensuring numerical stability, the implementation for this work also uses an alternative to
our recommendation of 
adding $\epsilon$ to $m_1$ and $m_2$ (see \apprefs{app:empirical_details} for details).
Our recommended strategy simplifies gain estimation but should not significantly
affect results, since usually
 ${\norms{\gbart}^2} \gg \epsilon$ in practice.

\figref{fig:gain_ratios} verifies the usefulness of our estimator by comparing
\algname{}'s estimates to improved (but impractical) ones.
Moving averages (i)~add robustness to estimation variance
and (ii)~incorporate multiple points from the optimization space.
For (ii), we ideally would use points from
independent training trajectories, but this is impractical.
Even so, \figref{fig:gain_ratios} suggests that
estimates from single and multiple trajectories can align closely.
We note that numerous prior works---for example, \citep{Schaul:2013,Kingma:2015,McCandlish:2018}---have relied on similar moving averages to
estimate gradient moments.

One final practical consideration is
the momentum parameter $\rho$
when using \algname{}
with momentum-SGD.
The performance
of momentum-SGD depends less critically 
on $\rho$
than the learning rate \citep{Shallue:2019}.
For this reason,
\algname{} often performs well
if $\rho$ remains constant across scales and iterations.
This approach to momentum scaling has also succeeded
in prior works involving
the linear scaling rule \citep{Goyal:2017,Smith:2018}.

\begin{figure*}[t]
\begin{center}
\small{
\begin{tabular}{
    @{}
    >{\raggedleft\arraybackslash}p{0.21in}
    @{\hspace{0.01in}}
    c
    @{}
    >{\raggedleft\arraybackslash}p{0.26in}
    @{\hspace{0.01in}}
    c
    @{}
    >{\raggedleft\arraybackslash}p{0.26in}
    @{\hspace{0.01in}}
    c
    @{}
    >{\raggedleft\arraybackslash}p{0.26in}
    @{\hspace{0.01in}}
    c
    @{}
    }
\multicolumn{8}{c}{
    \imagenet{}
} \\[-0.25em]
\rotatebox{90}{\hspace{0.16in} Val. Acc (\%)} &
\includegraphics[width=1.25739in]{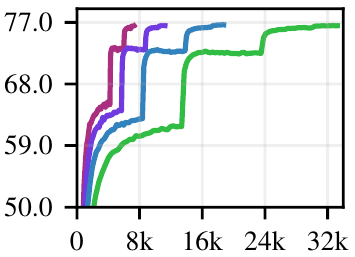} &
\rotatebox{90}{\hspace{0.16in} Val. Acc (\%)} &
\includegraphics[width=1.25739in]{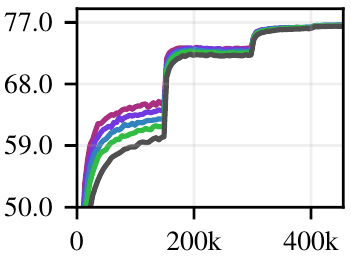} &
\rotatebox{90}{\hspace{0.16in} Training loss} &
\includegraphics[width=1.2076in]{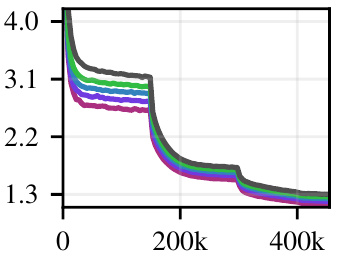} &
\rotatebox{90}{\hspace{0.10in} Learning rate $\etat$} &
\includegraphics[width=1.2076in]{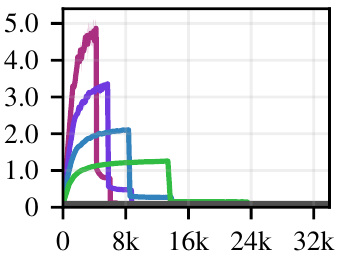} \\[-0.23em]
& {Iteration $t$} &
& {$S$-invariant iteration $\taut$} & 
& {$S$-invariant iteration $\taut$} & 
& {Iteration $t$} \\[-0.1em]
\multicolumn{8}{c}{
    \deepspeech{}
} \\[-0.25em]
\rotatebox{90}{\hspace{0.11in} Val. WAcc (\%)} &
\includegraphics[width=1.25739in]{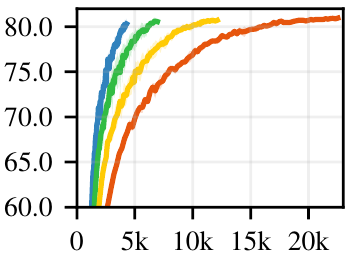} &
\rotatebox{90}{\hspace{0.11in} Val. WAcc (\%)} &
\includegraphics[width=1.25739in]{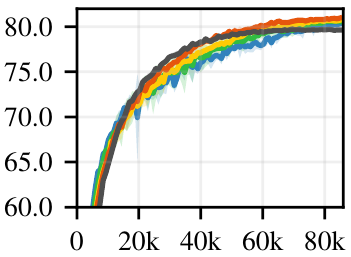} &
\rotatebox{90}{\hspace{0.16in} Training loss} &
\includegraphics[width=1.1827in]{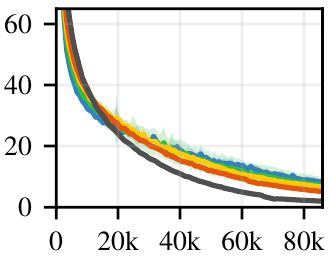} &
\rotatebox{90}{\hspace{0.10in} Learning rate $\etat$} &
\includegraphics[width=1.2076in]{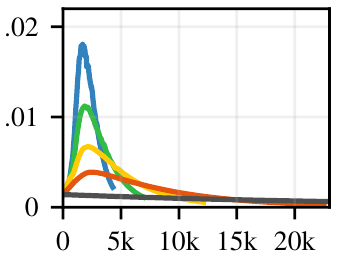} \\[-0.23em]
& {Iteration $t$} &
& {$S$-invariant iteration $\taut$} & 
& {$S$-invariant iteration $\taut$} & 
& {Iteration $t$} \\[-0.1em]
\multicolumn{8}{c}{
    \transformer{}
} \\[-0.45em]
\rotatebox{90}{\hspace{0.2in} Val. BLEU} &
\includegraphics[width=1.25739in]{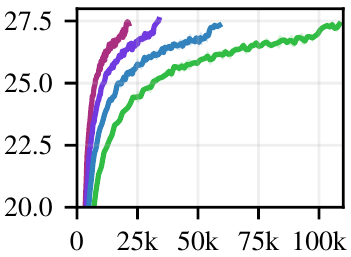} &
\rotatebox{90}{\hspace{0.2in} Val. BLEU} &
\includegraphics[width=1.25739in]{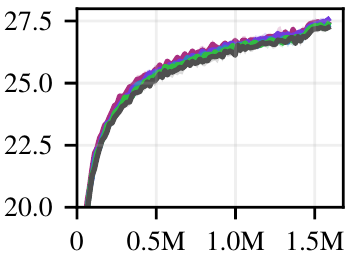} &
\rotatebox{90}{\hspace{0.16in} Training loss} &
\includegraphics[width=1.2076in]{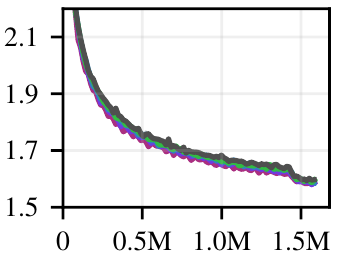} &
\rotatebox{90}{\hspace{0.10in} Learning rate $\etat$} &
\includegraphics[width=1.2076in]{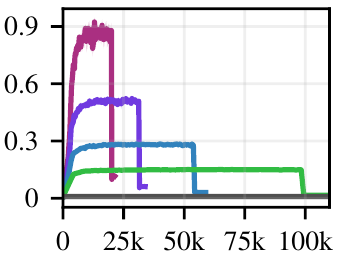} \\[-0.23em]
& {Iteration $t$} &
& {$S$-invariant iteration $\taut$} & 
& {$S$-invariant iteration $\taut$} & 
& {Iteration $t$} \\[-0.1em]
\multicolumn{8}{c}{
    \yolo{}
} \\[-0.25em]
\rotatebox{90}{\hspace{0.15in} Val. mAP (\%)} &
\includegraphics[width=1.25739in]{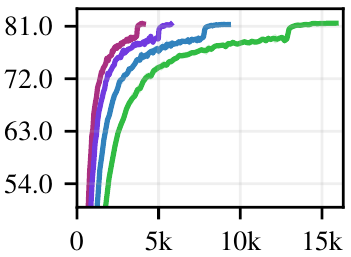} &
\rotatebox{90}{\hspace{0.15in} Val. mAP (\%)} &
\includegraphics[width=1.25739in]{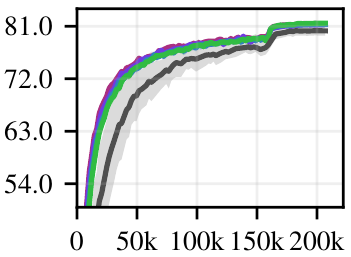} &
\rotatebox{90}{\hspace{0.16in} Training loss} &
\includegraphics[width=1.2076in]{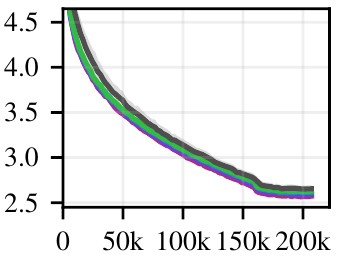} &
\rotatebox{90}{\hspace{0.10in} Learning rate $\etat$} &
\includegraphics[width=1.2076in]{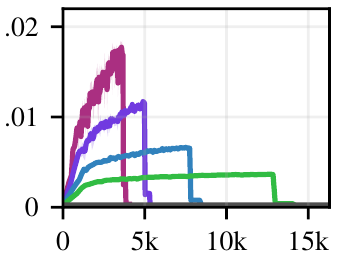} \\[-0.23em]
& {Iteration $t$} &
& {$S$-invariant iteration $\taut$} & 
& {$S$-invariant iteration $\taut$} & 
& {Iteration $t$} \\[-0.1em]
\\[-1.1em]
\multicolumn{8}{@{}c@{}}{
\includegraphics[width=5.5in]{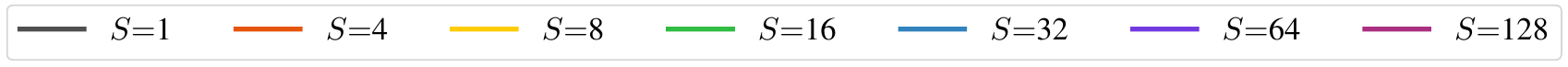} 
} \\[-1em]
\end{tabular}
}
\end{center}
\caption{\textbf{\algname{} training curves.}
For many scales and benchmarks,
\algname{}
trains quality models.
Training curves align closely in terms of $\taut$.
In all cases, $\etat$
warms up gradually at the start of training, even
though all $\lrschedname$ schedules are simple exponential or step decay functions (which are non-increasing in $t$).
}
\label{fig:invariance_curves}       
\vspf
\end{figure*}

\section{Empirical evaluation}
\label{sec:empirical}

We evaluate 
\algname{} on five practical training benchmarks.
We consider a variety
of tasks, models \citep{He:2016,He:2016b,Amodei:2016,Vaswani:2017,Redmon:2018}, and datasets \citep{Deng:2009,Krizhevsky:2009,Everingham:2010,Panayotov:2015}.
\tblref{tbl:tasks}
summarizes these training benchmarks.
We provide additional implementation details in
\apprefs{app:empirical_details}.

For each benchmark, we use
\emph{one simple learning rate
schedule}.
Specifically,
$\lrschedname$ is an
exponential decay function for
\cifar{} and \deepspeech{}, and a
step decay function otherwise.
We use standard $\lrschedname$ parameters for \imagenet{} and \yolo{}.
Otherwise,
we use tuned parameters
that approximately maximize the validation metric
(to our knowledge, there are no standard schedules
for solving 
\deepspeech{} and \transformer{}
    with momentum-SGD).
We use momentum $\rho = 0.9$ except for \transformer{},
in which case we use $\rho = 0.99$ for greater training stability.

\begin{figure*}[t]
\begin{center}
\small{
\begin{tabular}{
    @{}
    >{\raggedleft\arraybackslash}p{0.2in}
    @{\hspace{0.01in}}
    c
    @{}
    >{\raggedleft\arraybackslash}p{0.2in}
    @{\hspace{0.01in}}
    c
    @{}
    >{\raggedleft\arraybackslash}p{0.2in}
    @{\hspace{0.01in}}
    c
    @{}
    >{\raggedleft\arraybackslash}p{0.2in}
    @{\hspace{0.01in}}
    c
    @{}
    }
\rotatebox{90}{\hspace{0.28in} Scale $S$} &
\includegraphics[width=1.327in]{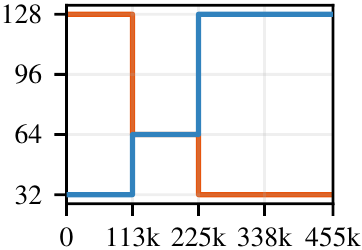} &
\rotatebox{90}{\hspace{0.16in} Val. Acc (\%)} &
\includegraphics[width=1.178in]{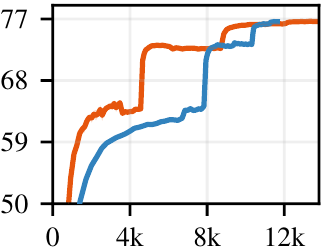} &
\rotatebox{90}{\hspace{0.16in} Val. Acc (\%)} &
\includegraphics[width=1.265in]{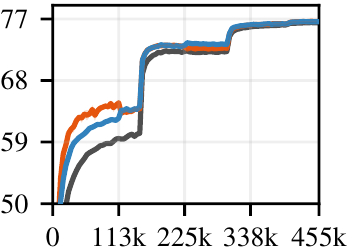} &
\rotatebox{90}{\hspace{0.16in} Training loss} &
\includegraphics[width=1.2898in]{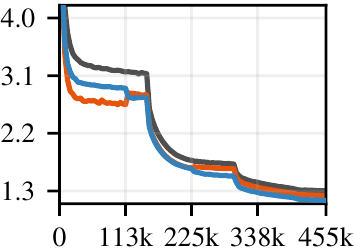} \\[-0.23em]
& {$S$-invariant iteration $\taut$} & 
& {Iteration $t$} &
& {$S$-invariant iteration $\taut$} & 
& {$S$-invariant iteration $\taut$}  \\[-0.1em]    
\multicolumn{8}{@{}c@{}}{
\includegraphics[width=3.1in]{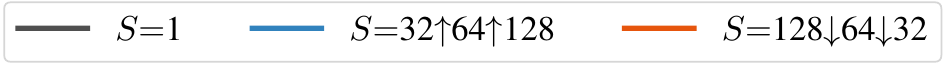} 
} \\[-1em]
\end{tabular}
}
\end{center}
\caption{\textbf{Elastic AdaScaling.}
For \imagenet{}, \algname{} scales training successfully
even with abrupt changes to $S$ (at $\taut= $ 133k, 225k).
Unlike \algname{},
LSW degrades model quality in this setting (see \tblref{tbl:results_summary}).
}
\label{fig:elastic_curves}       
\vspf
\end{figure*}

\begin{table*}[!b]
\vsp
\vsp
\vspf
\caption{\textbf{Comparison of final model quality.}
\emph{Shorthand:} AS=\algname{}, LSW=Linear scaling 
with warm-up, \LSS{}=Linear scaling with warm-up and additional steps,
\gr{gray}=model quality significantly worse than for $S = 1$ (5 trials, 0.95 significance), N/A=training diverges, Elastic$\uparrow$/$\downarrow$=elastic scaling with increasing/decreasing scale (see \figref{fig:elastic_curves}).
Linear scaling leads to poor model quality
as the scale increases;
\algname{} preserves model performance for nearly all cases.
}
\label{tbl:results_summary}
\vspf
\small{
\begin{center}
\begin{tabular}{
l
>{\raggedleft\arraybackslash}p{0.39in}
>{\raggedleft\arraybackslash}p{0.45in}
>{\raggedleft\arraybackslash}p{0.4in}
>{\raggedleft\arraybackslash}p{0.34in}
>{\raggedleft\arraybackslash}p{0.34in}
>{\raggedleft\arraybackslash}p{0.34in}
>{\raggedleft\arraybackslash}p{0.34in}
>{\raggedleft\arraybackslash}p{0.34in}
>{\raggedleft\arraybackslash}p{0.34in}
>{\raggedleft\arraybackslash}p{0.34in}
>{\raggedleft\arraybackslash}p{0.34in}
}
\toprule
\multirow{2}{*}{Task} &
\multirow{2}{*}{$S$} &
\multirow{2}{*}{\makecell{Total \\ batch size}}
& \multicolumn{3}{c}{ Validation metric}
& \multicolumn{3}{c}{Training loss} 
& \multicolumn{3}{c}{Total iterations}\\
\cmidrule(lr){4-6}
\cmidrule(lr){7-9}
\cmidrule(lr){10-12}
& & & 
AS & LSW & \LSS{} &
AS & LSW & \LSS{} &
AS & LSW & \LSS{}
\\
\midrule

\cifar{} & 1 & 128 & 94.1 & 94.1 & 94.1 & 0.157 & 0.157 & 0.157 & 39.1k & 39.1k & 39.1k\\
% & 4 & 512 & 94.1 & 93.9 & 0.154 & 0.159 & 10.7k & 9.77k \\
 & 8 & 1.02k & {94.1} & 94.0 & 94.0 & 0.153 & \gr{0.161} & 0.145 &  5.85k & {4.88k} & 5.85k\\
 & 16 & 2.05k & {94.1} & \gr{93.6} & 94.1 & 0.150 & \gr{0.163} & 0.136 & 3.36k & {2.44k}  & 3.36k\\
 & 32 & 4.10k & {94.1} & \gr{92.8} &  94.0 & 0.145 & \gr{0.177} & 0.128 &  2.08k & {1.22k} & 2.08k \\
 & 64 & 8.19k & \gr{93.9} & \gr{76.6} & \gr{93.0} & 0.140 & \gr{0.272} & 0.140 & 1.41k & {611} & 1.41k\\
\midrule

\imagenet{} & 1 & 256 & 76.4 & 76.4 & 76.4  & 1.30 & 1.30 & 1.30 & 451k & 451k & 451k\\
& 16 & 4.10k & 76.5 & \gr{76.3} & 76.5 & 1.26 & 1.31 & 1.27 & 33.2k & 28.2k & 33.2k \\
& 32 & 8.19k & 76.6 & \gr{76.1} & 76.4 & 1.23 & \gr{1.33} & 1.24 & 18.7k & 14.1k & 18.7k\\
& 64 & 16.4k & 76.5 & \gr{75.6} & 76.5 & 1.19 & \gr{1.35} & 1.20 & 11.2k & 7.04k & 11.2k\\
& 128 & 32.8k & 76.5 & \gr{73.3} & \gr{75.5}  & 1.14 & \gr{1.51} & 1.14 & 7.29k & 3.52k & 7.29k\\
& Elastic$\uparrow$ & various & 76.6 & \gr{75.7} & -- & 1.15 & \gr{1.37} & -- & 11.6k & 7.04k & -- \\
& Elastic$\downarrow$ & various & 76.6 & \gr{74.1} & -- & 1.23 & \gr{1.45} & -- & 13.6k & 9.68k & -- \\

\midrule

\deepspeech{} & 1 & 32 & 79.6 & 79.6 & 79.6 & 2.03 & 2.03 & 2.03 & 84.8k & 84.8k & 84.8k\\
& 4 & 128 & 81.0 & 80.9 & 81.0 & \gr{5.21} & \gr{4.66} & \gr{4.22} & 22.5k & 21.2k & 22.5k\\
& 8 & 256 &  80.7 & 80.2 & 80.7 & \gr{6.74} & \gr{6.81} & \gr{6.61} & 12.1k &  10.6k & 12.1k\\
& 16 & 512 & 80.6 & \gr{N/A} & \gr{N/A} & \gr{7.33} & \gr{N/A} & \gr{N/A} & 6.95k & 5.30k & 6.95k\\
& 32 & 1.02k & 80.3 & \gr{N/A} & \gr{N/A} & \gr{8.43} & \gr{N/A} & \gr{N/A} & 4.29k & 2.65k & 4.29k\\

\midrule

\transformer{} & 1 & 2.05k & 27.2 & 27.2 & 27.2 & 1.60 & 1.60 & 1.60 & 1.55M & 1.55M & 1.55M\\
 & 16 & 32.8k & 27.4 & 27.3 & 27.4 & 1.60 & 1.60 & 1.59 & 108k & 99.0k & 108k \\
 & 32 & 65.5k & 27.3 & \gr{27.0} & 27.3 & 1.59 & \gr{1.61} & 1.59 & 58.9k & 49.5k & 58.9k\\
 & 64 & 131k & 27.6 & \gr{26.7} & \gr{27.1} & 1.59 & \gr{1.63} & 1.60 & 33.9k & 24.8k & 33.9k\\
 & 128 & 262k & 27.4 & \gr{N/A} & \gr{N/A} & 1.59 & \gr{N/A} & \gr{N/A} & 21.4k & 12.1k & 21.4k\\

\midrule

\yolo & 1 & 16 & 80.2 & 80.2 & 80.2 & 2.65 & 2.65 & 2.65 & 207k & 207k & 207k\\
& 16 & 256 & 81.5 & 81.4 & 81.9 & 2.63 & 2.66 & 2.47 & 15.9k & 12.9k & 15.9k \\
& 32 & 512 & 81.3 & {80.5} & 81.7 & 2.61 & \gr{2.81} & 2.42 & 9.27k & 6.47k & 9.27k \\
& 64 & 1.02k & 81.3 & \gr{70.1} & 80.6 & 2.60 & \gr{4.02} &  2.51 & 5.75k & 3.23k  & 5.75k\\
& 128 & 2.05k & 81.4 & \gr{N/A} &  \gr{N/A} & 2.57 & \gr{N/A} & \gr{N/A} & 4.07k & 1.62k & 4.07k\\

\bottomrule

\end{tabular}
\end{center}
}
\end{table*}

\figref{fig:invariance_curves} (and \figref{fig:cifar10_intro}) contains \algname{} training curves
for the benchmarks and many scales.
Each curve plots the mean of five distributed training
runs with varying random seeds.
As $S$ increases, \algname{}
trains for fewer iterations and consistently preserves model quality.
Interestingly,
the training curves align closely
when plotted in terms of scale-invariant iterations.

For $S > 1$, \algname{}'s learning rate
increases gradually during initial training,
despite the fact that $\lrschedname$ is non-increasing.
Unlike warm-up, 
this behavior emerges naturally from a principled algorithm,
not hand-tuned user input.
Thus, \algname{} provides not only a compelling alternative to warm-up
but also a plausible explanation for warm-up's success.

For \imagenet{}, we also consider elastic scaling.
Here, the only change to \algname{} is that $S$
changes abruptly after some iterations.
We consider two cases: (i)~$S$ increases
from 32 to 64 at $\tau_t = \TB / 4$
and from 64 to 128 at $\tau_t = \TB / 2$,
and (ii)~the scale decreases at the same points,
from 128 to 64 to 32.
    In \figref{fig:elastic_curves},
we include training curves
from this setting.
Despite the abrupt batch size changes, \algname{} trains quality models,
highlighting
\algname{}'s value for the common
scenario of dynamic resource availability.

As a baseline for all benchmarks, we also evaluate linear scaling with warm-up (LSW).   
As inputs, LSW takes single-batch schedule $\Olrschedname = \lrschedname$
and duration $\OT = \TB$,
where $\lrschedname$ and $\TB$ are the inputs to \algname{}.
Our warm-up implementation closely follows that of \citet{Goyal:2017}.
LSW trains for $\ceil{\OT/S}$ iterations,
applying warm-up to the first
    5.5\% of iterations.
During warm-up, the learning rate increases
linearly
from $\Olrschedname(0)$ to $S \cdot \Olrschedname(0)$.

Since LSW trains for fewer iterations than \algname{}, we also consider a stronger
baseline, \LSS{}, which
matches \algname{} in total iterations.  \LSS{} uses the same learning
rate schedule as LSW except scaled (stretched) along the iterations axis by
the difference in training duration.
We note \LSS{} is \emph{significantly less practical} than LSW and \algname{}, since it requires
either (i)~first running \algname{} to determine the number of iterations, or (ii)~tuning
the number of iterations.

\tblref{tbl:results_summary} compares results
for \algname{}, LSW, and \LSS{}.
LSW consistently trains for fewer iterations,
but doing so comes at a cost.
As $S$ grows larger,
LSW consistently degrades model quality and
sometimes diverges.
For these divergent cases,
we also tested doubling the warm-up duration to 11\% of iterations,
and training still diverged.
Similarly, even with the benefit of additional iterations,
\LSS{} also produces worse model quality in many cases.
In contrast, \algname{}
preserves model quality for nearly all cases.

\begin{figure*}
\begin{center}
\small{
\begin{tabular}{
    @{}
    >{\raggedleft\arraybackslash}p{0.1in}
    @{}
    >{\raggedleft\arraybackslash}p{0.1in}
    @{\hspace{0.05in}}
    c
    @{\hspace{0.1in}}
    c
    @{\hspace{0.1in}}
    c
    @{\hspace{0.1in}}
    c
    @{\hspace{0.1in}}
    c
    @{}}
&& \hspace{0.08in} Single batch & \hspace{0.08in} \algname{} & \hspace{0.08in} LSW & \hspace{0.08in} Scaled SGD \\[-0.1em]
&& \hspace{0.08in} $S$$=$1, $T$$=$39.1k & \hspace{0.08in} $S$$=$16, $\TB$$=$39.1k & \hspace{0.08in} $S$$=$16, $\TSo$$=$39.1k & \hspace{0.08in} $S$$=$16, $T$$=$3.28k \\[-0.2em]
\rotatebox{90}{\hspace{0.25in} Total $\lrschedname$} &
\rotatebox{90}{\hspace{0.242in} decrease} &
\includegraphics[width=1.22486in]{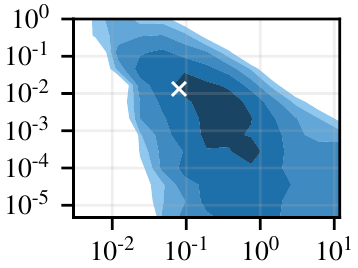} &
\includegraphics[width=1.22486in]{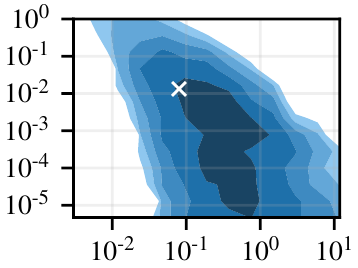} &
\includegraphics[width=1.22486in]{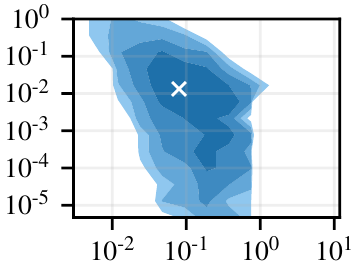} &
\includegraphics[width=1.189in]{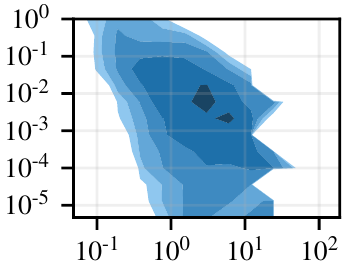} &
\raisebox{0.14in}{\includegraphics[width=0.4162in]{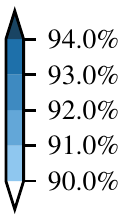}} \\[-0.1em]
&& \hspace{0.08in} Initial $\lrschedname$ 
& \hspace{0.08in} Initial $\lrschedname$ 
& \hspace{0.08in} Initial $\lrschedname$ 
& \hspace{0.08in} Initial $\lrschedname$ 
\\[-1em]
\end{tabular}
}
\end{center}
\caption{\textbf{\algname{} results for many learning rate schedules.}
Heat maps cover the space of exponential decay $\lrschedname$ schedules for \cifar{}.
At scale 16, validation accuracies for \algname{}
align closely with results for single-batch training,
with the space of 94+\% schedules growing moderately with \algname{}.
With LSW, no schedule achieves 94\% accuracy.
On the right, direct $\lrschedname$ search at scale 16 
produces inferior results to \algname{} (here the total iterations,
3.28k, is the average total iterations among 94+\% \algname{} trials).
Thus, \algname{} induces a superior family of
schedules for large-batch training.
The white `$\times$' indicates the $\lrschedname$ used for \figref{fig:cifar10_intro}.
}
\label{fig:many_scheds}       
\vspf
\end{figure*}

As a final comparison,
\figref{fig:many_scheds}
demonstrates
\algname{}'s
performance on \cifar{} with many different $\lrschedname$ schedules.
We consider
a 13$\times$13
grid of exponential decay schedules
and plot contours 
of the resulting validation accuracies.
At scale 16,
\algname{}  results
align with results for single-batch training,
illustrating that \algname{}
preserves model quality for many schedules.
Moreover, \algname{}
convincingly outperforms direct search over exponential
decay schedules for scaled SGD at $S$$=$16.
This suggests that 
\algname{}
provides a better learning rate family for distributed training.

\section{Relation to prior work}
\label{sec:related}

While linear scaling with warm-up is perhaps the most popular fixed scaling
rule, researchers have considered a few alternative strategies.
``Square root learning rate scaling''
\citep{Krizhevsky:2014,Li:2014,Hoffer:2017,You:2018} 
multiplies
learning
rates by the square root of the batch size increase.
Across scales, this preserves the covariance of the SGD update.
Establishing this invariant remains poorly justified, however,
and often root scaling degrades model quality in practice \citep{Goyal:2017,Golmalt:2018,Jastrzebski:2018}.
\algname{} adapts learning rates by making $\etat \E{\norms{\gbart}^2}$ invariant
across scales, which results in our bound from \secref{sec:theory}.
\citet{Shallue:2019} compute
near-optimal parameters for many tasks
and scales, and
the results do not
align with any fixed rule.
To ensure effective training,
the authors recommend avoiding such rules
and
re-tuning parameters for each new scale.
This solution is inconvenient and resource-intensive,
however, and \citeauthor{Shallue:2019}
do not consider adapting learning rates to the state of training.

Many prior works have also considered the role of gradient
variance in SGD.
\citet{Yin:2018} provide conditions---including sufficiently small
batch size and sufficiently large gradient variance---under which
linear learning rate scaling works well.
\citeauthor{Yin:2018} do not provide an alternative strategy
for adapting learning rates when linear scaling fails.
\citet{McCandlish:2018} 
study the impact of gradient variance
on scaling efficiency.  By averaging the relative
gradient variance over the course of training,
they make rough (yet fairly accurate) estimates
of training time complexities as a function of scale.
While \citet{McCandlish:2018} do not provide an algorithm for obtaining
such speed-ups, 
these general findings
also relate to \algname{},
since gradient variance similarly determines
\algname{}'s efficiency.
Much like \algname{},
\citet{Johnson:2018} also adapt learning
rates to lower amounts of gradient variance---in this case
when using SGD with importance sampling.
Because the variance reduction is relatively small in this setting, however,
distributed training can have far greater impact on
training times.
Lastly,
many algorithms
also adapt
to gradient moments
for improved training, given a single batch size---see \citep{Schaul:2013,Kingma:2015,Balles:2018}, just to name a few.
In contrast,
\algname{} translates learning rates for
one batch size into learning rates
for a larger scale.
Perhaps future versions of \algname{} will combine approaches and achieve both goals.
\citet{You:2017,You:2020} scaled training to large batch sizes
by combining adaptive
gradient algorithms with scaling rule heuristics.

\section{Discussion} 
\label{sec:discussion}

SGD is not perfectly parallelizable.
Unsurprisingly, the linear scaling rule can fail
at large scales.
In contrast, \algname{} accepts
sublinear speedups in order to better preserve model quality.
What do the speed-ups from \algname{} tell us about the 
scaling efficiency of SGD in general?
For many problems, such as \imagenet{} 
with batch size 32.8k,
\algname{} provides lower bounds on SGD's scaling efficiency.
An important remaining question is whether \algname{} is close to optimally efficient,
or if other practical algorithms can achieve similar model quality
with fewer iterations.

\algname{} establishes a useful new parameterization of learning rate schedules for large-batch
SGD.
Practitioners can provide a simple $\lrschedname$ schedule,
which \algname{} adapts to learning rates for scaled training.
From this,
warm-up behavior emerges naturally, which 
produces quality models for many problems and scales.
Even in elastic scaling settings, \algname{}
adapts successfully to the state of training.
Given these appealing qualities,
it seems important to further study 
this family of learning rate schedules.

Based on our empirical results,
as well as the algorithm's practicality
and theoretical justification,
we believe \algname{} can be
very valuable for distributed training.

\section*{Acknowledgements}

For valuable feedback,
we thank 
Emad Soroush,
David Dai,
Wei Fang,
Okan Akalin,
Russ Webb,
and Kunal Talwar.

\bibliography{references}
\bibliographystyle{icml2020}

\nocite{Kloeden:1992}
\nocite{Zhang:2018}
\nocite{Zhang:2019b}
\newpage
\appendix 
\onecolumn

\section{Proofs} \label{app:proofs}

In this appendix, we prove the 
results from \secref{sec:theory} and \secref{sec:gain_scaling}.
We first prove a lemma in \secref{app:key_lemma},
which we apply in the proofs.
We prove \thmref{thm:base_bound}
in \secref{app:proof_base_bound}, 
\thmref{thm:gs_bound} in \secref{app:proof_gs_bound}, 
and \thmref{thm:ls_bound} in \secref{app:proof_ls_bound}.
We also prove \propref{prop:identity} in \secref{app:identity_prop_proof}
and \propref{prop:linear} in \secref{app:linear_prop_proof}.

\subsection{Key lemma} \label{app:key_lemma}

\begin{lem} \label{lem:key_lemma}
Given Assumptions~\ref{ass:pl}, \ref{ass:smooth}, \ref{ass:variance} 
and \mbox{$\eta \in (0, 2 \beta^{-1})$},
define
$\gamma = \eta \alpha (2 - \eta \beta)$
and \mbox{$\DeltaB = \tfrac{1}{2 \gamma} \eta^2 \beta V$}.
Consider \algref{alg:gs} with $\lrsched{t} = \eta$.
For all iterations $t$, we have
\[
\E{F(\wt)-F^*} \leq 
\left[F(\wz)-F^*\right] 
\textstyle \prod_{t'=0}^{t-1}(1 - \rtp \gamma) 
+ \DeltaB \, .
\]
\end{lem}
\begin{proof}
We prove this by induction. 
To simplify notation, let us define $\Ft(\w) = F(\w) - F^*$.
For $t = 0$, we have
\[ \Es{\Ft(\wz)} = \Ft(\wz) \leq \Ft(\wz)\textstyle \prod_{t'=0}^{-1}(1 - \rtp \gamma) +  \DeltaB \, . \]
Here we are using the convention $\prod_{i=0}^{-1} x_i = 1$.
For $t \geq 1$, assume the inductive hypothesis
\begin{equation} \label{eqn:base_induction}
\Es{\Ft(\wtm)} \leq 
\Ft(\wz) 
\prod_{t'=0}^{t-2}(1 - \rtp \gamma) 
+ \DeltaB \, .
\end{equation}

Applying \assref{ass:smooth} (smoothness) and the update equation $\wt = \wtm - \rtm \eta \gtm$, we have
\begin{align}
\Ft(\wt) & \leq \Ft(\wtm) + \ip{ \nabla F(\wtm), \wt - \wtm} + \tfrac{\beta}{2} \norm{ \wt - \wtm}^2 \nonumber \\
& = \Ft(\wtm) - \rtm \eta \ip{\nabla F(\wtm), \gtm} + \rtm^2 \etasq \tfrac{\beta}2 \norm{ \gtm }^2 \, . \nonumber
\end{align}

Taking the expectation with respect to the $S$ random batches from step $t$, we have
\[
\E{ \Ft(\wt) \mid \wtm} \leq \Ft(\wtm) - \rtm \eta \norm{\nabla F(\wtm)}^2 + \rtm^2 \eta^2 \tfrac{\beta}{2} \E{\norm{\gtm}^2 \mid \wtm} \, .
\]

Now taking the expectation with respect to the distribution of $\wtm$, it follows that
\begin{equation}
\E{ \Ft(\wt) } \leq \E{\Ft(\wtm)} - \rtm \eta \E{\norm{\nabla F(\wtm)}^2} + \rtm^2 \eta^2 \tfrac{\beta}{2} \E{\norm{\gtm}^2 } \, .
\label{eqn:bc1}
\end{equation}

For the last term, we have
\begin{align}
\E{\norm{\gtm}^2 } &= \E{\norm{(\gtm - \nabla F(\wtm)) + \nabla F(\wtm)}^2 } \nonumber \\
&= \E{\norm{\gtm - \nabla F(\wtm)}^2 + \norm{\nabla F(\wtm)}^2} \nonumber \\
&= \E{\tfrac{1}{S} \VBtm + \norm{\nabla F(\wtm)}^2} \nonumber \\
&= \frac{1}{\rtm} \E{ \VBtm + \norm{\nabla F(\wtm)}^2} \nonumber \\
&\leq \frac{1}{\rtm} \left( \E{\norm{\nabla F(\wtm)}^2} + V \right) \, . \label{eqn:bc2}
\end{align}
Combining \eqnref{eqn:bc2} with \eqnref{eqn:bc1}, we have
\begin{align}
\E{ \Ft(\wt) } & \leq \E{\Ft(\wtm)} - \rtm \eta (1 - \eta \tfrac{\beta}{2}) 
\E{\norm{\nabla F(\wtm)}^2} + \rtm \eta^2 \tfrac{\beta}2 V \nonumber \\
& \leq (1 - \rtm \gamma) \E{\Ft(\wtm)} + \rtm \gamma \Delta \, .
\end{align}
In this last step, we applied \assref{ass:pl} (PL condition) and plugged in
definitions for $\gamma$ and $\Delta$.

To complete the proof, we apply \eqnref{eqn:base_induction}:
\begin{align*}
\Es{\Ft(\wt)} & \leq (1 - \rtm \gamma) \left(\Ft(\wz)\prod_{t'=0}^{t-2}(1 - \rtp \gamma) + \DeltaB \right) + \rtm \gamma \Delta \\
&= \Ft(\wz)\prod_{t'=0}^{t-1}(1 - \rtp \gamma) + \DeltaB \, .
\end{align*}
\end{proof}

\subsection{Proof of \thmref{thm:base_bound}} \label{app:proof_base_bound}

\basebound*

\begin{proof}
The theorem is a special case of \lemref{lem:key_lemma}.
In particular,
\algref{alg:sgd} with inputs $\lrsched{t} = \eta$, $S = 1$, and $T$ iterations
is equivalent to \algref{alg:gs} with $\TB = T$ and the same scale 
and learning rate inputs.
This follows from the fact that $\rt = 1$ for all
iterations of \algname{} when $S = 1$.
Thus, we can obtain the result by plugging 
$t = T$ into the bound from \lemref{lem:key_lemma}.
\end{proof}

\subsection{Proof of \thmref{thm:gs_bound}} \label{app:proof_gs_bound}

\gsbound*

\begin{proof}
Let $T$ denote the total iterations for \algref{alg:gs}.
Applying \lemref{lem:key_lemma}, we have
\begin{equation}
\E{F(\wt)-F^*} \leq (F(\wz)-F^*)\prod_{t'=0}^{T-1}(1 - \rtp \gamma)  + \DeltaB  \, .
\label{eqn:inq1}
\end{equation}
Now note that for any $r \geq1$ and $ x \in [0, 1]$, we have 
\begin{equation}
1-rx \leq (1-x)^r. 
\label{eqn:a8sjd}
\end{equation}

This holds because 
for any $r \geq1$ and $ x \in [0, 1]$,
the function $(1-x)^r$ is convex in $x$,
    and $1-rx$ is tangent to this function at $x=0$. 
Thus,
\begin{equation}
\prod_{t'=0}^{T-1}(1 - \rtp \gamma) \leq (1-\gamma)^{\sum_{t'=0}^{T-1}r_{t'}} \, .
\label{eqn:inq3}
\end{equation}
Note that this requires $1 - \rt \gamma \geq 0$ for all $t$, which is true because
$\rt \leq S \leq \gamma^{-1}$.
Now plugging \eqnref{eqn:inq3} into \eqnref{eqn:inq1},
\begin{align*}
\E{F(\wt)-F^*} &\leq (F(\wz)-F^*)(1-\gamma)^{\sum_{t'=0}^{T-1}r_{t'}} + \DeltaB \\
&= (F(\wz)-F^*)(1-\gamma)^{\bar{r}T} + \DeltaB \, .
\end{align*}
\end{proof}

\subsection{Proof of \thmref{thm:ls_bound}} \label{app:proof_ls_bound}

\lsbound*

\begin{proof}
We reduce the theorem to a special
case of \thmref{thm:base_bound}.
Define $\xtil = (\xtil^{(1)}, \ldots, \xtil^{(S)})$,
where $\xtil^{(i)} \sim \X$ for each $i \in [S]$, and $\xtil^{(1)}, \ldots, \xtil^{(S)}$ are jointly independent.
Denote by $\tilde{\X}$ the distribution of $\xtil$.
Also define
\[
\ftil(\w, \xtil) = \frac{1}S \sum_{i=1}^S f(\w, \xtil^{(i)}) \, .
\]
It follows that for any $\w$,
\[
\Eu{\xtil}{\norms{\nabla \ftil(\w, \xtil) - \nabla F(\w)}^2} = \frac{1}S \VB \leq \frac{V}S \, .
\]
The algorithm described in \thmref{thm:ls_bound}
is identical to running \algref{alg:sgd} with
scale $1$, batch distribution $\tilde{\X}$, loss $\ftil$, learning rate $\lrsched{t} = S \eta$, and variance upper bound $\frac{V}S$.
Plugging these values into \thmref{thm:base_bound}, we have
\begin{align*}
\E{F(\wT)-F^*} & \leq 
(1 - S \eta \alpha (2 - S \eta \beta))^T [F(\wz)-F^*] + \frac{S \eta \beta V S^{-1}}{2 \alpha \left( 2 - S \eta \beta \right)}  \\
&= 
\left(1 - S \gamma \cdot \left( \tfrac{2 - S \eta \beta}{2 - \eta \beta} \right)\right)^T
[F(\wz)-F^*] 
+ 
\left( \tfrac{2 - \eta \beta}{2 - S \eta \beta} \right) \DeltaB 
\\
&\leq 
\left(1 - \gamma \cdot \left( \tfrac{2 - S \eta \beta}{2 - \eta \beta} \right)\right)^{ST}
[F(\wz) -F^*] 
+
\left( \tfrac{2 - \eta \beta}{2 - S \eta \beta} \right) \DeltaB \, .
\end{align*}
The last step follows from \eqnref{eqn:a8sjd}.
\end{proof}

\subsection{Proof of \propref{prop:identity}} \label{app:identity_prop_proof}

\identityprop*

\begin{proof}
Since the gradient variance is zero,
the $\mathtt{compute\_gradient}$ function returns
$\nabla F(\wt)$, which does not depend on $S$.
Thus, $\w^{(S)} = \w^{(1)}$ and $F(\w^{(S)}) = F(\w^{(1)})$.
\end{proof}

\subsection{Proof of \propref{prop:linear}} \label{app:linear_prop_proof}

\linearprop*

\begin{proof}
The scaled SGD algorithm runs for ${\nu T}/S$ iterations
and
follows the update rule
\[
\wtp = \wt - \tfrac{S {\eta}}{\nu} \nabla F(\wt) + \tfrac{S {\eta}}\nu \xibt \, .
\]
Here $\xibt$ is normally distributed with $\E{\xibt} = \m{0}$
and $\mathrm{cov}(\xibt, \xibt) = \tfrac{\nu}{S} \bm{\Sigma}$.
In the limit $\nu \rightarrow +\infty$, this difference equation
converges to a stochastic differential equation 
on the interval $[0, {\eta} {T}]$ \citep[Chapter~9]{Kloeden:1992}:
\[
d \w = -\nabla F(\w) dt + ({\eta} \bm{\Sigma})^{1/2} d \m{W}(t) \,, \quad \text{where} \quad \m{W}(t) \sim \mathcal{N}(\m{0}, \m{I}) \, .
\]
Since this SDE does not depend on $S$, the distributions of
$\w^{(S)}$ and $\w^{(1)}$ are identical 
in this limit.
Thus, we have $\Es{F(\w^{(S)})} = \Es{F(\w^{(1)})}$.
\end{proof}

\section{Additional details on empirical comparisons}
\label{app:empirical_details}

This appendix provides additional details of our experiment set-up. 

\subsection{Learning rate schedules}

We describe the $\lrschedname$ schedules for each training
benchmark in \tblref{tbl:learning_rates}.
We use two learning rate families: exponential decay and
step decay.
Using parameters $\eta_0$, $d$, and $w_i$,
we define
\mbox{$\lrsched{t} = \eta_0 d^{(t/\TSo)}$} for exponential decay families and
$\lrsched{t} = \eta_0 d^{\sum_i^n\mathbbm{1}[{t > w_i}]}$ for step decay families.
Here $\TSo$ denotes the total iterations for scale $S = 1$.
Note that in all cases, we use simple schedules and no warm-up.

\begin{table}[h]
\caption{\textbf{Learning rate schedules for training benchmarks.}
}
\label{tbl:learning_rates}
\small{
\begin{center}
\begin{tabular}{lllll}
\toprule
Benchmark & Learning rate famliy & 
$\eta_0$ & $d$ & $w_i$ \\
\midrule
\cifar{} & Exponential decay & 0.08 & 0.0133 & N/A  \\
\imagenet{} & Step decay & 0.1 & 0.1 & 150,240, 300,480, 400,640  \\
\deepspeech{} & Exponential decay & 1.4 $\times 10^{-3}$ & 0.05 & N/A \\
\transformer{} & Step decay & 0.01 & 0.1 & 1,440,000 \\
\yolo{} & Step decay & 2.5 $\times 10^{-4}$ & 0.1 & 160,000, 180,000 \\
\bottomrule
\end{tabular}
\end{center}
}
\end{table}

For \imagenet{} and \yolo{},
we used standard learning rate schedules from \citep{Goyal:2017}
and \citep{Zhang:2019b}.
For \cifar{}, \deepspeech{}, and \transformer{},
we chose learning rate parameters, via hand-tuning, that approximately
maximized model quality.
This was necessary for \deepspeech{}
and \transformer{}, since our
reference implementations train with the
Adam optimizer \citep{Kingma:2015}, and
momentum-SGD requires different learning rate values.
%\footnote{.
%For \deepspeech{}, we used the reference implementation.
%available at.
%\url{https://github.com/tensorflow/models/tree/master/research/deep_speech}..
%For \transformer{}, the address is.
%\url{https://github.com/tensorflow/models/tree/master/official/transformer}..
%}

\subsection{Warm-up implementation}

Our
warm-up procedure closely follows the strategy of \citet{Goyal:2017}.
We apply warm-up
for the first 5.5\% of training iterations---we
    denote this number by $W_S$.
During warm-up, the learning rate increases
linearly, starting at the 
initial learning rate for single-batch training and finishing at $S$
times this value.
After warm-up, we apply linear scaling to the single-batch schedule.
Following \citet{Goyal:2017},
we modify this scaled schedule so that
the total iterations, including warm-up, is proportional to $S^{-1}$.
For step-decay schedules,
we omit the first $W_S$ iterations after warm-up.
For exponential decay schedules, we compress the scaled schedule by $W_S$ iterations,
using slightly faster decay.

\subsection{Numerical stability strategy for AdaScale implementation} \label{app:numerical_stability}

As discussed in \secref{sec:estimating_noise},  
the AdaScale implementation for our empirical results uses an alternative strategy
for ensuring numerical stability instead of the recommendation from \secref{sec:estimating_noise}.
The recommended strategy is simpler, but the results should not differ significantly.

For the experiments, we estimate mean and variance quantities
separately.  
Algebraically, this gain estimator 
is equivalent to that of \secref{sec:estimating_noise}, except
for the numerical stability measures.
In particular, we define
\[ 
\begin{array}{c}
\\[-1.35em]
\textstyle \hat{\sigma}^2_t = \tfrac{1}{S - 1} \sum_{i = 1}^S \norms{\gt^{(i)}}^2
- \tfrac{S}{S - 1} \norm{\gbart}^2 \, , \\ \\[-0.9em]
\text{and}\quad \hat{\mu}^2_t = \norm{\gbart}^2 - \tfrac{1}S \hat{\sigma}^2_t \, .
\\
\\[-1.4em]
\end{array}
\]
Here $\hat{\sigma}^2_t$ and $\hat{\mu}^2_t$
are unbiased estimates of $\Eus{\wt}{\VBt}$ and $\Eus{\wt}{\MBt}$.
We estimate $\rt$ by plugging in moving averages
$\bar{\sigma}^2_t$
and $\bar{\mu}_t^2$, which average
$\hat{\sigma}^2_{t}$ and $\hat{\mu}^2_{t}$ over prior iterations.
The implementation uses exponential moving average parameter $\theta = \max\{1 - S / 1000, 0\}$, where $\theta = 0$
results in no averaging.
Before averaging, we set
$\hat{\sigma}^2_t \gets \mathrm{max}(\hat{\sigma}^2_t, \epsilon)$
(to prevent division by zero)
and $\hat{\mu}^2_t \gets \mathrm{max}(\hat{\mu}^2_t, 0)$ (to ensure $\rt \in [1, S]$).
To initialize,
we set $r_{0} \gets 1$,
and for iterations $t < (1 - \theta)^{-1}$,
we define
$\bar{\sigma}^2_t$ and $\bar{\mu}_t^2$
as the mean 
of past samples.
The averaging procedure does not factor $\hat{\sigma}^2_t$ and $\hat{\mu}^2_t$
into the estimate for $r_t$ (only mean and variance estimates from prior iterations),
but we recommend using estimates from the current iteration if possible.

\subsection{Benchmark-specific implementation details}

Here we describe implementation details
that are specific to each benchmark task.

\subsubsection{\cifar{}}

We train ResNet-18 (preactivation) models \citep{He:2016b},
using the standard training data split for CIFAR-10 \citep{Krizhevsky:2009}.
We use weight decay $= 5\times 10^{-4}$.
For batch normalization, we use parameters momentum $= 0.995$ and $\epsilon = 2 \times 10^{-5}$, and we do not train the batch normalization scaling parameters.
We apply standard data augmentation during training.
Specifically, we pad images to $40 \times 40$ and random crop to $32 \times 32$, and we also apply random horizontal reflections.

\subsubsection{\imagenet{}}
For ImageNet classification \citep{Deng:2009}, we train ResNet-50 models \citep{He:2016}. 
Our implementation closely follows the implementation of \citet{Goyal:2017}.  
We use stride-2 convolutions on $3 \times 3$~layers. 
For each block's final batch normalization layer, we initialize the batch norm scaling parameters to $0$ (and we initialize to $1$ everywhere else).
We use weight decay parameter $10^{-4}$. Since each GPU processes 128 examples per batch,
we use ghost batch normalization \citep{Hoffer:2017} with ghost batch size $32$.
We resize input images to $224 \times 224 \times 3$. For data augmentation, we apply random cropping and left-right mirroring during training.

\subsubsection{\deepspeech{}}
We use \citet{Amodei:2016}'s Deep Speech 2 model architecture. The model consists of two 2D convolutional input layers, five bidirectional RNN layers, one fully connected layer, and softmax outputs. Each convolutional layer has $32$ filters. The RNN layers use GRU cells with hidden size $800$.
We apply batch normalization to the inputs of each layer. 
The batch norm parameters are momentum $=0.997$ and $\epsilon=10^{-5}$. The loss is CTC loss.
We apply gradient clipping with threshold $100$.
The inputs to the network are log spectrograms, which we compute using $20$ms windows from audio waveforms sampled at $16$ kHz. The training data is the \texttt{train-clean-100} and \texttt{train-clean-360} partitions of the OpenSLR LibriSpeech Corpus, which amounts to 460 hours of recorded speech. We evaluate models on the \texttt{dev-clean} partition and use the CTC greedy decoder for decoding. 

\subsubsection{\transformer{}}
We train Transformer base models \citep{Vaswani:2017}. We use dynamic batching with at most $256$ tokens per example. 
In \tblref{tbl:results_summary}, the ``batch size'' is the maximum number of tokens processed per iteration. 
Our implementation closely follows that of \citet{Vaswani:2017}.  Unlike \citeauthor{Vaswani:2017}, we use only the final model for evaluation instead of the average of the last five checkpoints. 
We train on the WMT 2014 English-German dataset and evaluate on the \texttt{newstest2014} test set. We compute BLEU scores using the script in the official TensorFlow GitHub repository, \url{https://github.com/tensorflow/models/blob/master/official/transformer/compute_bleu.py}
\subsubsection{\yolo{}}
We train YOLOv3 models \citep{Redmon:2018}. To achieve
high mAP scores, we also apply mixup \citep{Zhang:2018} and class label smoothing, following \citep{Zhang:2019b}. We also use focal loss. We use batch normalization momentum$=0.9$ and weight decay $=5 \times 10^{-4}$. We resize input images to $416 \times 416$ (for both training and validation). 
%For evaluation, we use a non-maximum suppression threshold of $0.45$ and keep at most $150$ boxes after suppression. 
We report mAP values at IOU threshold $0.5$. We use the Pascal VOC 
2007 \texttt{trainval} and 2012 \texttt{trainval} datasets for training and the 2007 test set for validation \citep{Everingham:2010}. During training, we initialize the darknet-53 convolutional layers with weights trained on ImageNet.

\subsection{Miscellaneous}

In practice,
wall time speed-ups also depend on
system scaling efficiency.
Since most aspects of system scaling
relate orthogonally to the training algorithm,
we limit our scope to algorithmic aspects of training.

For
\figref{fig:many_scheds},
one dimension defines initial value 
$\lrsched{0}$,
and the second dimension specifies total
decrease $\lrsched{\TB}/\lrsched{0}$.
For single-batch training, we use $T = 39.1 \times 10^3$ steps.
We run \algname{}
and the LW baseline at $S = 16$,
and we compare the  final validation accuracies.

\section{Robustness to averaging parameter \label{app:theta}}

In this appendix, we test the robustness of \algname{}
to the averaging parameter $\theta$ for estimating gain ratios (see \secref{sec:estimating_noise}).
When $\theta = 0$,
\algname{} does not average estimates of gradient moments.
The closer $\theta$ is to $1$, the more that \algname{} averages
across iterations.

Using the \cifar{} benchmark, we compare four values of $\theta$
at scales $S = 8$ and $S = 32$.
The case $\theta = 1 - S / 1000$ corresponds
to the \cifar{} experiment for \figref{fig:cifar10_intro}.
We average the resulting metrics over five trials.
\figref{fig:theta} contains the training curves.

\begin{figure}[h]
\begin{center}
\small{
\begin{tabular}{
    @{}
    >{\raggedleft\arraybackslash}p{0.11in}
    @{\hspace{0.01in}}
    c
    @{}
    >{\raggedleft\arraybackslash}p{0.14in}
    @{\hspace{0.01in}}
    c
    @{}
    >{\raggedleft\arraybackslash}p{0.14in}
    @{\hspace{0.01in}}
    c
    @{}
    >{\raggedleft\arraybackslash}p{0.14in}
    @{\hspace{0.01in}}
    c
    @{}
    }
\multicolumn{8}{c}{
 $S = 8$
} \\
\rotatebox{90}{\hspace{0.17in} Val. Acc (\%)} &
\includegraphics[width=1.2767in]{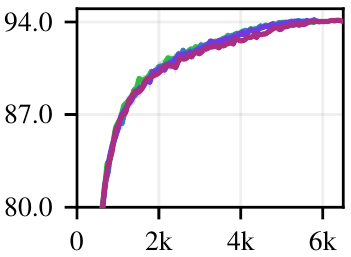} &
\rotatebox{90}{\hspace{0.11in} Train objective} &
\includegraphics[width=1.22617949in]{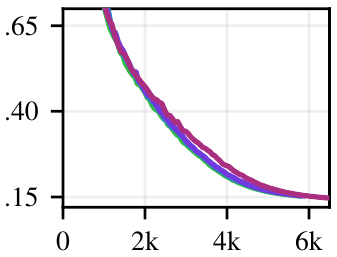} &
\rotatebox{90}{\hspace{0.31in} Gain $\rt$} &
\includegraphics[width=1.1503in]{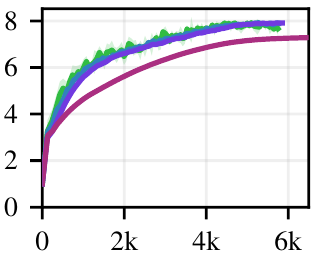} &
\rotatebox{90}{\hspace{0.10in} Learning rate $\etat$} &
\includegraphics[width=1.22617949in]{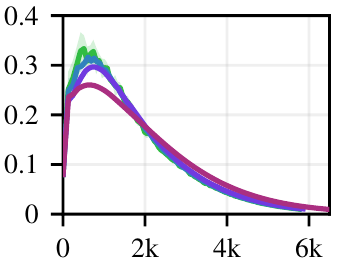} \\
& {Iteration $t$} &
& {Iteration $t$} &
& {Iteration $t$} &
& {Iteration $t$} \\[-0.1em]
\\[-1em]
\multicolumn{8}{@{}c@{}}{
\includegraphics[width=5.0in]{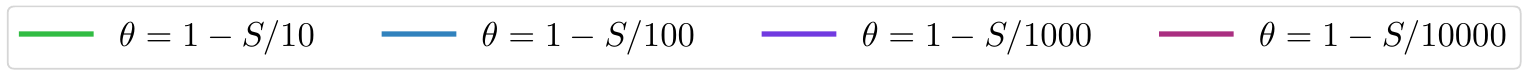} 
} \\
\multicolumn{8}{c}{
 $S = 32$
} \\
\rotatebox{90}{\hspace{0.17in} Val. Acc (\%)} &
\includegraphics[width=1.2767in]{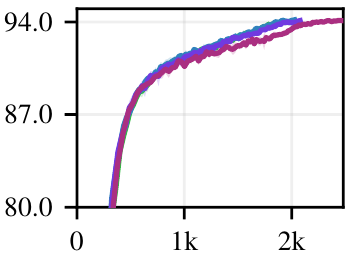} &
\rotatebox{90}{\hspace{0.11in} Train objective} &
\includegraphics[width=1.22617949in]{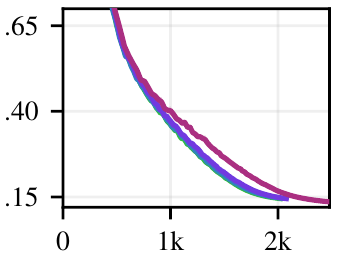} &
\rotatebox{90}{\hspace{0.31in} Gain $\rt$} &
\includegraphics[width=1.2in]{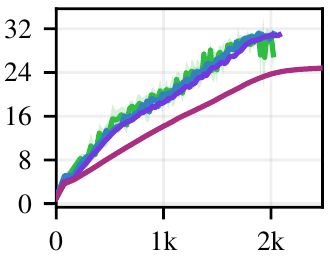} &
\rotatebox{90}{\hspace{0.10in} Learning rate $\etat$} &
\includegraphics[width=1.22617949in]{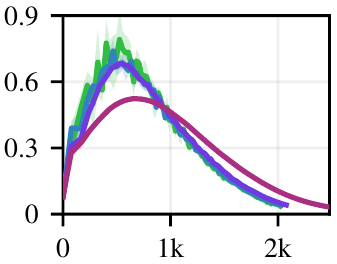} \\
& {Iteration $t$} &
& {Iteration $t$} &
& {Iteration $t$} &
& {Iteration $t$} \\[-0.1em]
\\[-1em]
\multicolumn{8}{@{}c@{}}{
\includegraphics[width=4.76in]{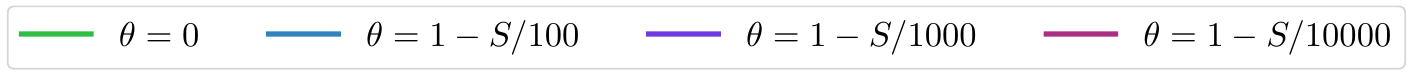} 
} \\[-0.75em]
\end{tabular}
}
\end{center}
\caption{\textbf{\algname{} training curves with varying moving average parameter.}
}
\label{fig:theta}       
\end{figure}

We also include final metric values in \tblref{tbl:theta}.

\begin{table}[h!]
\caption{\textbf{\algname{} final metrics with varying moving average parameter.}
}
\label{tbl:theta}
\small{
\begin{center}
\begin{tabular}{rrccc}
\toprule
$S$ & $\theta$ & \makecell{Final val. \\ accuracy (\%) }& \makecell{Final train \\ objective} & \makecell{Total \\ iterations} \\

\midrule

8 & $1-S/10$ & 94.0 & 0.153 & 5.75k \\
& $1-S/100$ & 94.1 & 0.154 & 5.78k \\
& $1-S/1000$ & 94.1 & 0.153 & 5.85k \\
& $1-S/10000$ & 94.1 & 0.147 & 6.45k \\

\midrule

32 & $0$ & 94.0 & 0.145 & 2.02k \\
& $1-S/100$ & 94.1 & 0.147 & 2.03k \\
& $1-S/1000$ & 94.1 & 0.145 & 2.08k \\
& $1-S/10000$ & 94.1 & 0.136 & 2.46k \\

\bottomrule
\end{tabular}
\end{center}
}
\end{table}

For the three smaller settings of $\theta$,
the results align very closely.
This suggests that \algname{}
is robust to the choice of $\theta$.
When $\theta = 1 - S / 10000$,
we see that smoothing more significantly biases
gain ratio estimates,
which leads to more contrasting results.

\section{Additional empirical results}
\label{app:additional_plots}

This appendix provides additional empirical results.

\subsection{Gain ratio estimation}

Our online gain ratio estimates align closely with offline estimates (computed by averaging over 1000 batches). 
\figref{fig:transformer_gain} demonstrates this for the \transformer{} task.

\begin{figure}
\begin{center}
\small{
\begin{tabular}{@{}p{0.12in}@{}c@{}c@{}}
& \hspace{0.0in} \transformer, $S$$=$16 &
\transformer, $S$$=$128 
 \\[-0.1em]
\rotatebox{90}{\hspace{0.17in} Gain ratio $\rt$}
& \includegraphics[width=1.148in]{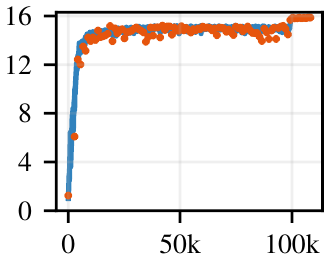} &
\includegraphics[width=1.196475in]{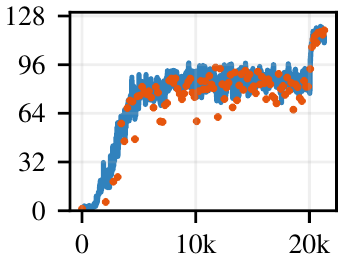} 
\\[-0.2em]
& \hspace{0.15in} Iteration $t$
& \hspace{0.15in} Iteration $t$
\\[-0.1em]
\multicolumn{3}{c}{
\includegraphics[width=1.71in]{plots/gain_ratios/legend.pdf} 
}
\end{tabular}
}
\end{center}
\vspace{-1.2em}
\caption{\textbf{Gain ratios for \transformer{}}.
Plots compare moving average $\rt$ estimates
to values computed offline (using 1000 batches).
}
\label{fig:transformer_gain}
\vspf
\end{figure}

\subsection{\cifar{} AdaScale training curves}

\figref{fig:cifar10_invariance_curves} shows additional plots for the \cifar{} task. Notably, the plots show training loss curves at various scales and full view of the learning rate curves.

\begin{figure}[t]
\begin{center}
\small{
\begin{tabular}{
    @{}
    >{\raggedleft\arraybackslash}p{0.11in}
    @{\hspace{0.01in}}
    c
    @{}
    >{\raggedleft\arraybackslash}p{0.14in}
    @{\hspace{0.01in}}
    c
    @{}
    >{\raggedleft\arraybackslash}p{0.14in}
    @{\hspace{0.01in}}
    c
    @{}
    >{\raggedleft\arraybackslash}p{0.14in}
    @{\hspace{0.01in}}
    c
    @{}
    }
\rotatebox{90}{\hspace{0.16in} Val. Acc (\%)} &
\includegraphics[width=1.25739in]{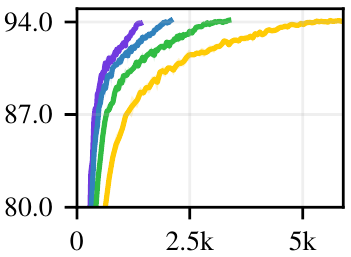} &
\rotatebox{90}{\hspace{0.16in} Val. Acc (\%)} &
\includegraphics[width=1.25739in]{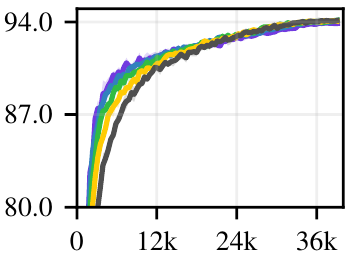} &
\rotatebox{90}{\hspace{0.16in} Training loss} &
\includegraphics[width=1.2076in]{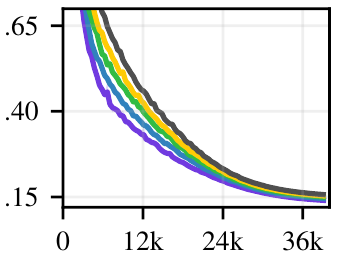} &
\rotatebox{90}{\hspace{0.10in} Learning rate $\etat$} &
\includegraphics[width=1.2076in]{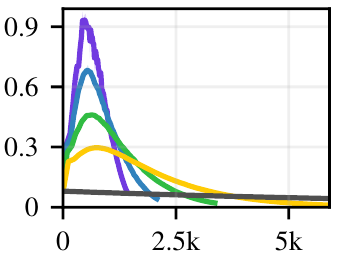} \\[-0.23em]
& {Iteration $t$} &
& {$S$-invariant iteration $\taut$} & 
& {$S$-invariant iteration $\taut$} & 
& {Iteration $t$} \\[-0.1em]
\multicolumn{8}{@{}c@{}}{
\includegraphics[width=5.5in]{plots/colors-legend.pdf} 
} \\[-1em]
\end{tabular}
}
\end{center}
\caption{\textbf{\algname{} training curves for \cifar{}.}
\algname{}
trains quality models at various scales.
}
\label{fig:cifar10_invariance_curves}       
\vspf
\end{figure}

%\subsection{Elastic scaling}

%Learning rate and gain ratio curves for the two dynamic scaling scenarios we consider (discussed in \secref{sec:empirical}) align surprisingly well with the corresponding curves for the scenarios where the scale is kept constant throughout the training. This is shown in \figref{fig:elastic_learning}. Gain ratio adapts quickly to the abrupt change in scale, allowing the algorithm to preserve model quality. 
%\input{figures/elastic_learning}

\end{document}